\documentclass[journal]{IEEEtran}

\usepackage{cite}
\usepackage[pdftex]{graphicx}
\usepackage{amsmath}
\usepackage{array}
\usepackage{booktabs}
\usepackage{color}
\usepackage{graphics}
\usepackage{epsfig}
\usepackage{amssymb}
\usepackage{amsthm}
\usepackage{subfigure}
\usepackage[table]{xcolor}
\usepackage{stfloats}
\usepackage{caption}
\usepackage{threeparttable}
\usepackage{multirow}
\usepackage{multicol}
\usepackage{cite}
\usepackage{bm}
\newtheorem{definition}{$\bf {Definition}$}
\newtheorem{lemma}{$\bf {Lemma}$}

\newtheorem{proposition}{$\bf {Proposition}$}
\usepackage[switch]{lineno}

\begin{document}
\title{Learnable Faster Kernel-PCA for Nonlinear Fault Detection: Deep Autoencoder-Based Realization}

\author{\IEEEauthorblockN{Zelin Ren, Xuebing Yang, Yuchen Jiang and Wensheng Zhang }
\thanks{Corresponding author: Wensheng Zhang.}

\thanks{Z. Ren and W. Zhang are with the Research Center of Precision
Sensing and Control, Institute of Automation, Chinese Academy of Sciences,
Beijing, 100190, China, and University of Chinese Academy of Sciences,
Beijing, 101408, China (e-mail:rzl8816@126.com; zhangwenshengia@hotmail.com).}
}


\maketitle

\begin{abstract}
Kernel principal component analysis (KPCA) is a well-recognized nonlinear dimensionality reduction method that has been widely used in nonlinear fault detection tasks. As a kernel trick-based method, KPCA inherits two major problems. First, the form and the parameters of the kernel function are usually selected blindly, depending seriously on trial-and-error. As a result, there may be serious performance degradation in case of inappropriate selections. Second, at the online monitoring stage, KPCA has much computational burden and poor real-time performance, because the kernel method requires to leverage all the offline training data. In this work, to deal with the two drawbacks, a learnable faster realization of the conventional KPCA is proposed. The core idea is to parameterize all feasible kernel functions using the novel nonlinear DAE-FE (deep autoencoder based feature extraction) framework and propose DAE-PCA (deep autoencoder based principal component analysis) approach in detail. The proposed DAE-PCA method is proved to be equivalent to KPCA but has more advantage in terms of automatic searching of the most suitable nonlinear high-dimensional space according to the inputs. Furthermore, the online computational efficiency improves by approximately 100 times compared with the conventional KPCA. With the Tennessee Eastman (TE) process benchmark, the effectiveness and superiority of the proposed method is illustrated.
\end{abstract}


\section{Introduction}
\label{sec:introduction}
\IEEEPARstart{F}{ault} detection plays an important role in maintaining normal operating conditions and ensuring system safety, which has attracted extensive attention in recent years \cite{gaotie2015, pengmpe2013, luoitm2018, pengtie2016}. Efficient fault detection can greatly strengthen the system reliability and reduce unexpected economic loss \cite{jiangisj2021}. With the increasing complexity of industrial processes, data-driven fault detection and process monitoring techniques take the dominant position in the field of fault detection and have become a research hotspot \cite{han2021tcyb, yutcyb, kodamana2019tcst, niutcyb, fazai2019eaai, naderiauto2017}. How to excavate the potential system information for different fault detection tasks from the massively available sensor data remains challenging.

Kernel principal component analysis (KPCA) is a typical and well-known data-driven method for dimensionality reduction based on kernel trick \cite{scholkopf1996kernel, geces2009}. Owing to its good nonlinear feature extraction capability, KPCA and its extensions have been widely used for various fault detection and multivariate statistical process monitoring (MSPM) tasks of industrial processes currently \cite{jiang2019iecr, yin2015tie, zhang2018n}. Simmini {\it et al.} proposed a self-tuning KPCA method to detect the faults in chiller systems with a high accuracy \cite{simmini2021tcst}. Fan {\it et al.} presented a fast incremental nonlinear matrix completion method, which can make KPCA successfully monitor nonlinear processes when there exist missing data \cite{fan2021tii}. Deng {\it et al.} integrated principal component analysis (PCA) and KPCA and presented a serial PCA that has the ability to better exploit the underlying process's structure in order to enhance fault diagnosis performance \cite{deng2018tnnls}.

Although the kernel method has the ability to monitor nonlinear systems and detect faults by observing the fluctuations of features in high-dimensional space, it still has several problems and drawbacks. First, the choice of kernel function has a great influence on the detection performance of the method. The forms of kernel function are multitudinous, and once a specific form of kernel function is determined, such as radial basis function (RBF) kernel, that means a subset of nonlinear mapping functions is determined. When its kernel parameters are chosen by offline testing (such as $\sigma $ in RBF kernel), a specific nonlinear mapping is finally determined. It can be seen that this process has blindness, because the choice of kernel function totally depends on artificial selection and the kernel parameters are selected by trial and error. When there is no domain knowledge/expert experience, blind selection of kernel parameters will greatly limit the optimal selection of high-dimensional space. If the domain knowledge/expert experience is unsuitable for the working condition, it will also lead to unsatisfactory kernel parameters for fault detection. Thus, how to automatically search a proper kernel parameter is an intractable issue. Second, the online detection of the kernel method requires all offline samples. It causes a huge amount of calculation and has a serious influence when having demand on timely detection. In order to tackle the problem of low-efficiency online calculation existing in kernel methods, it is necessary to propose a new approximate nonlinear mapping method to replace current kernel methods.

To deal with the drawbacks of the existing KPCA-based methods, we consider reforming the DAE network to realize a learnable and faster KPCA by designing a network structure. Deep autoencoder (DAE) is a basic and popular data-driven deep learning method \cite{suntcyb, sakurada2014mlsda, ahmed2021tpami, creswell2019}. Due to its outstanding performance on powerful nonlinear representation for mass data, the potential of DAE has been demonstrated in industry. Recently, many researchers have presented a large number of variants, which have been prosperously applied to the fault detection and diagnosis area \cite{tang2021isa, liutcyb, yu2020tcst, jang2020tii, zhang2018jpc, ren2020tii}. How to adopt DAE to efficiently serve various fault detection tasks, instead of kernel methods, has become the recent tendency.

In this article, we handle the aforementioned problems existing in kernel method by proposing a DAE-FE (deep autoencoder based feature extraction) framework, which includes an encoder module, a linear feature extraction (LFE) module, and a decoder module. We prove that the DAE-FE framework is essentially equivalent to a learnable and faster kernel method. Further, based on the DAE-FE framework, a DAE-PCA (deep autoencoder based principal component analysis) method is proposed to transform PCA into the corresponding nonlinear approach. The difference between KPCA and DAE-PCA is that the process that KPCA transforms a linear approach into the corresponding nonlinear version is composed of two stages. The first stage is to determine kernel function by manual selection. The second stage is to apply the determined kernel function to convert the linear method into a nonlinear one. Therefore, for the kernel method, the setting of the kernel parameters is completely artificial and independent of the following linear method. The effect of nonlinearity depends entirely on artificial setting and has blindness. As for DAE-PCA method, it is an end-to-end learning process. When the network structure is determined, the encoder module and decoder module, as nonlinear modules, are learned together with the linear module in each back-propagation. The benefit of DAE-PCA is that the nonlinear modules serve the linear module and automatically search for a more suitable higher-dimensional space, which KPCA cannot achieve. In the detailed design of DAE-PCA, a PCA module is devised to achieve the function of PCA dimensionality reduction based on Cayley Transform \cite{cayley1846}. In addition, in order to enhance the representation ability of the extracted system features, an extra regularization term is introduced to reduce the variance of the feature samples in each dimension, which is beneficial for nonlinear fault detection. The key contributions of our DAE-PCA can be summarized as follows.

\begin{enumerate}
\item[1)] A DAE-FE framework is proposed for nonlinear fault detection. Based on the DAE-FE framework, a DAE-PCA method is further proposed, which is theoretically proven to be equivalent to a learnable KPCA. Compared with KPCA, the DAE-PCA method can automatically acquire the weight parameters of the network through training, unlike KPCA that selects kernel parameters blindly, depending seriously on trial-and-error.

\item[2)] The proposed DAE-PCA method presents an alternative faster implementation of KPCA. In comparison with the conventional KPCA, the DAE-PCA method costs less time for online detection since its online detection time do not rely on the amount of training data, which greatly meets the requirement of real-time detection.

\item[3)] To obtain a better orthogonal load matrix, the idea of Cayley Transform is introduced to the design of the DAE-PCA method. Moreover, a regularization term is presented for the DAE-PCA method in order to constrain extracted system features to have a smaller variance on each feature dimension. It makes the feature samples more compact and have a better representation capability, which greatly improves the fault detection precision and stability.	

\end{enumerate}

The rest of this article is organized as follows. In Section II, KPCA and DAE are briefly illustrated. Section III first presents the DAE-FE framework. After that, the proposed implementation of DAE-PCA is presented for nonlinear fault detection. In Section IV, the effectiveness and feasibility of the proposed approach are demonstrated in the Tennessee Eastman (TE) process benchmark. Finally, Section V concludes this article.

\section{Preliminaries}
\subsection{KPCA}
Suppose that there are $m$ sensors in an industrial system and each sensor collects $N$ samples, the input variable matrix can be denoted as $\mathbf{X}$, i.e., \begin{center}$\mathbf{X} = [{\mathbf{x}_1},{\mathbf{x}_2},...,{\mathbf{x}_N}]^{\top} \in {\mathbb{R}^{N \times m}}.$\end{center}

As a typical linear self-supervised model, the traditional PCA \cite{wise1990pca} model achieves dimensionality reduction by the principle of capturing the maximum variance of the input data, and it can be formulated by minimizing the reconstruction error:
\begin{equation}
\begin{array}{l}
	\mathop {\min }\limits_\mathbf{P} \left\| {\mathbf{X} - \mathbf{X}\mathbf{P}{\mathbf{P}^{\top}}} \right\|_F^2\\
	s.t.{\kern 1pt} {\kern 1pt} {\kern 1pt} {\mathbf{P}^{\top}}\mathbf{P} = {\mathbf{I}_a},
\end{array}
\end{equation}
where $\mathbf{P} \in {\mathbb{R}^{m \times a}}$ is the orthogonal load matrix, $a$ denotes the feature dimension after PCA operation, and ${\left\| {  \cdot } \right\|_F}$ represents $F$-norm. Furthermore, we can obtain the dimensionality-reduction feature $\mathbf{T} \in {\mathbb{R}^{N \times a}}$, i.e., the score matrix, which can be calculated by
\begin{equation}
\mathbf{T} = {\left[ {{\mathbf{t}_1}, \ldots ,{\mathbf{t}_N}} \right]^{\top}} = \mathbf{X}\mathbf{P} \in \mathbb{R}^{N \times a},
\end{equation}

KPCA is a nonlinear PCA method by introducing the kernel method. Kernel method aims to map the original process variables into a high-dimension reproducing kernel Hilbert space (RKHS) through a proper nonlinear mapping function, thus solving linearly inseparable problems \cite{scholkopf1996kernel}. Given a nonlinear mapping function $\phi \left( {\cdot} \right)$, it maps the original samples ${\mathbf{x}_i}\left( {i = 1,2, \cdots ,N} \right)$ into a RKHS $\mathcal{H}$, i.e., ${\mathbf{x}_i} \in {\mathbb{R}^m} \to \phi \left( {{\mathbf{x}_i}} \right) \in {\mathbb{R}^d}$, where $d$ is the dimension of $\mathcal{H}$. Thus, $\mathbf{X}$ is mapped into a nonlinear feature matrix $\mathbf{\Phi} $,
\begin{center}$
	\mathbf{\Phi}  = {\left[ {\phi \left( {{\mathbf{x}_1}} \right),\phi \left( {{\mathbf{x}_2}} \right), \cdots ,\phi \left( {{\mathbf{x}_N}} \right)} \right]^{\top}} \in {\mathbb{R}^{N \times d}}.$\end{center}
Notice that $\mathbf{\Phi}$ needs to be centralized to be zero mean,
\begin{equation}
	\begin{aligned}
		\mathbf{\bar \Phi} =  {\left[ {\overline \phi  \left( {{\mathbf{x}_1}} \right),\overline \phi  \left( {{\mathbf{x}_2}} \right), \ldots ,\overline \phi  \left( {{\mathbf{x}_N}} \right)} \right]^{\top}} =   {\mathbf{H}_e}\mathbf{\Phi}.
	\end{aligned}
\label{eq:He}
\end{equation}
where ${\mathbf{H}_e}  = \left( {{\mathbf{I}_N} - {N^{ - 1}}{\mathbf{1}_N}\mathbf{1}_N^{\top}} \right)$ is the centering matrix and ${\mathbf{1}_N} = {\left[ {1,1, \ldots ,1} \right]^{\top}} \in {\mathbb{R}^N}$.

Accordingly, the optimization problem of KPCA can be written as
\begin{equation}
\begin{array}{l}
	\mathop {\min }\limits_\mathbf{P} \left\| {\mathbf{\bar \Phi}  - \mathbf{\bar \Phi}\mathbf{P}{\mathbf{P}^{\top}}} \right\|_F^2\\
	s.t.{\kern 1pt} {\kern 1pt} {\kern 1pt} {\mathbf{P}^{\top}}\mathbf{P} = {\mathbf{I}_a},
\end{array}
\label{eq:KPCA}
\end{equation}
where $\mathbf{P} \in {\mathbb{R}^{d \times a}}$ is the orthogonal load matrix. Then, the score matrix $\mathbf{T}$ is computed by $\mathbf{T} = \mathbf{\bar \Phi}\mathbf{P}$.

When finishing model establishment, PCA or KPCA will map the original space into a principal component subspace ($\mathcal{PS}$) and a residual subspace ($\mathcal{RS}$), thereby achieving fault detection in these two subspaces.

\subsection{DAE}
The DAE can be trained to extract nonlinear system features by minimizing the reconstruction error of the input in an unsupervised self-learning manner. The main advantage of DAE is reconstructing the input data with fewer errors and effectively extracting the nonlinear feature. In DAE, the structure is made up of two parts: an encoder and a decoder. The encoder aims to extract the system features by nonlinear mapping of the input, while the decoder is used to reconstruct the input.

Let $\mathbf{X}$ and $\mathbf{\hat X} \in {\mathbb{R}^{N \times m}}$ be the input and output of DAE. The encoder and the decoder are respectively represented by nonlinear mappings $En\left(  \cdot  \right)$ and $De\left(  \cdot  \right)$, of which the nonlinearity is expressed by nonlinear activation functions. Formally, we have
\begin{equation}
\mathbf{\hat X} = De\left( {En\left( \mathbf{X} \right)} \right).
\label{eq:Xfs}
\end{equation}

The loss function for training DAE can be expressed by
\begin{equation}
loss = \mathcal{L}\left( {\mathbf{X},\mathbf{\hat X}} \right) + \Omega,
\end{equation}
where $\mathcal{L}\left(  \cdot  \right)$ is the reconstruction loss and $\Omega $ denotes regularization terms for specific problems or prior knowledge. The common choice of $\mathcal{L}\left(  \cdot  \right)$ is the sum of square errors,
\begin{equation}
\mathcal{L}\left( {\mathbf{X}, \mathbf{\hat X}} \right) = {\left\| {\mathbf{X} - \mathbf{\hat X}} \right\|_F^2}.
\end{equation}

When the network has finished training, the extracted nonlinear system feature $\mathbf{T} \in {\mathbb{R}^{N \times a}}$ and the reconstruction error $\mathbf{\tilde X} \in {\mathbb{R}^{N \times m}}$ can be acquired by
\begin{equation}
\mathbf{T} = En\left( \mathbf{X} \right),
\end{equation}
\begin{equation}
\mathbf{\tilde X} = \mathbf{X} - \mathbf{\hat X}.
\label{eq:xrs}
\end{equation}

For fault detection tasks, $\mathbf{T}$ and $\mathbf{\tilde X}$ can be utilized to monitor the $\mathcal{PS}$ and $\mathcal{RS}$, respectively.


\section{Methodology}
\subsection{DAE-FE Framework}
The proposed DAE-FE framework is to transform linear feature extraction methods to their nonlinear versions using neural networks. The whole DAE-FE network structure is illustrated in Fig. \ref{fig:DAE-FE}, which consists of three modules: the encoder module, the LFE module, and the decoder module. The encoder and decoder modules aim to realize the transformation from the original input space $ \mathcal{X}$ and RKHS $ \mathcal{H}$ respectively. The purpose of the LFE module is to utilize a neural network to achieve the function of LFE, thereby obtaining nonlinear system features in $ \mathcal{H}$. LFE is usually solved by matrix operation. In order to have the same form as other modules, it is necessarily transformed into the structure of neural network. In this way, the LFE module can combine the encoder module and the decoder module to jointly realize an end-to-end solution through back-propagation.

\subsubsection{Encoder Module}
It utilizes $En\left(  \cdot  \right)$ to map the original input $ \mathbf{ X }\in {\mathbb{R}^{N \times m}}$ from the input domain $\mathcal{X}$ into a nonlinear feature space $ \mathcal{H}$ to obtain the feature matrix $\mathbf{\Phi}  = En\left(\mathbf{ X } \right) \in {\mathbb{R}^{N \times d}}$.

It will be proven in Proposition 1 that $En\left(  \cdot  \right)$ is in essence a learnable nonlinear function that can construct a kernel function. To describe the relationship between $En\left(  \cdot  \right)$ and kernel method, we need to give the following definition and lemma, of which the relevant proofs can be referred to the literature \cite{mohri2018book}.

\begin{definition}
Given the $N$-dimensional input domain $\mathcal{X} \subset {\mathbb{R}^N}$, a function $K: \mathcal{X} \times \mathcal{X} \to \mathbb{R}$ is called a kernel over $\mathcal{X}$ which is expressed by
\begin{center}
$\forall \mathbf{x},\mathbf{x}' \in \mathcal{X},{\kern 1pt} {\kern 1pt} {\kern 1pt} {\kern 1pt} {\kern 1pt} {\kern 1pt} K\left( {\mathbf{x},\mathbf{x}'} \right) = \left\langle {\phi \left(\mathbf{x} \right),\phi \left( {\mathbf{x}'} \right)} \right\rangle  = \phi {\left(\mathbf{x} \right)^{\top}}\phi \left( {\mathbf{x}'} \right),$
\end{center}
where $\left\langle  \cdot  \right\rangle $ denotes inner product operation and $\phi \left(  \cdot  \right)$ is a nonlinear mapping function.
\end{definition}

\begin{lemma}
Let kernel $K: \mathcal{X} \times \mathcal{X} \to \mathbb{R}$ be a continuous and symmetric function. Then, $K$ is said to be a positive definite symmetric (PDS) kernel if for any $\left\{ {{\mathbf{x}_1}, \ldots ,{ \mathbf{x}_N}} \right\} \subseteq \mathcal{X}$, the Gram matrix $\mathbf{K} = {\left[ {K\left( {{\mathbf{x}_i},{\mathbf{x}_j}} \right)} \right]_{i,j}} \in {\mathbb{R}^{N \times N}}$ is symmetric positive semidefinite (SPSD).
\end{lemma}

\begin{lemma}
Let $K: \mathcal{X} \times \mathcal{X} \to \mathbb{R}$ be a PDS kernel. Then, there exists a Hilbert space $ \mathcal{H}$ and a mapping $\phi \left(  \cdot  \right)$ from $ \mathcal{X}$ to $ \mathcal{H}$ such that:
\begin{center}$\forall \mathbf{x},\mathbf{x}' \in \mathcal{X},{\kern 1pt} {\kern 1pt} {\kern 1pt} {\kern 1pt} {\kern 1pt} {\kern 1pt} K\left( {\mathbf{x},\mathbf{x}'} \right) = \left\langle {\phi \left(\mathbf{x} \right),\phi \left( {\mathbf{x}'} \right)} \right\rangle.$ \end{center}
Furthermore, $ \mathcal{H}$ is called a RKHS associated with $K$ if it has the following reproducing property:
\begin{center}$\forall h \in \mathcal{H},{\kern 1pt} \forall \mathbf{x} \in \mathcal{X},{\kern 1pt} {\kern 1pt} {\kern 1pt} {\kern 1pt} {\kern 1pt} {\kern 1pt} {\kern 1pt} h\left( x \right) = \left\langle {h,K\left( {\mathbf{x}, \cdot } \right)} \right\rangle . $ \end{center}
\end{lemma}

\begin{figure}[!t]
	\setlength{\abovecaptionskip}{0pt}
	\setlength{\belowcaptionskip}{0pt}
	\centering
	\includegraphics[width=0.99\columnwidth]{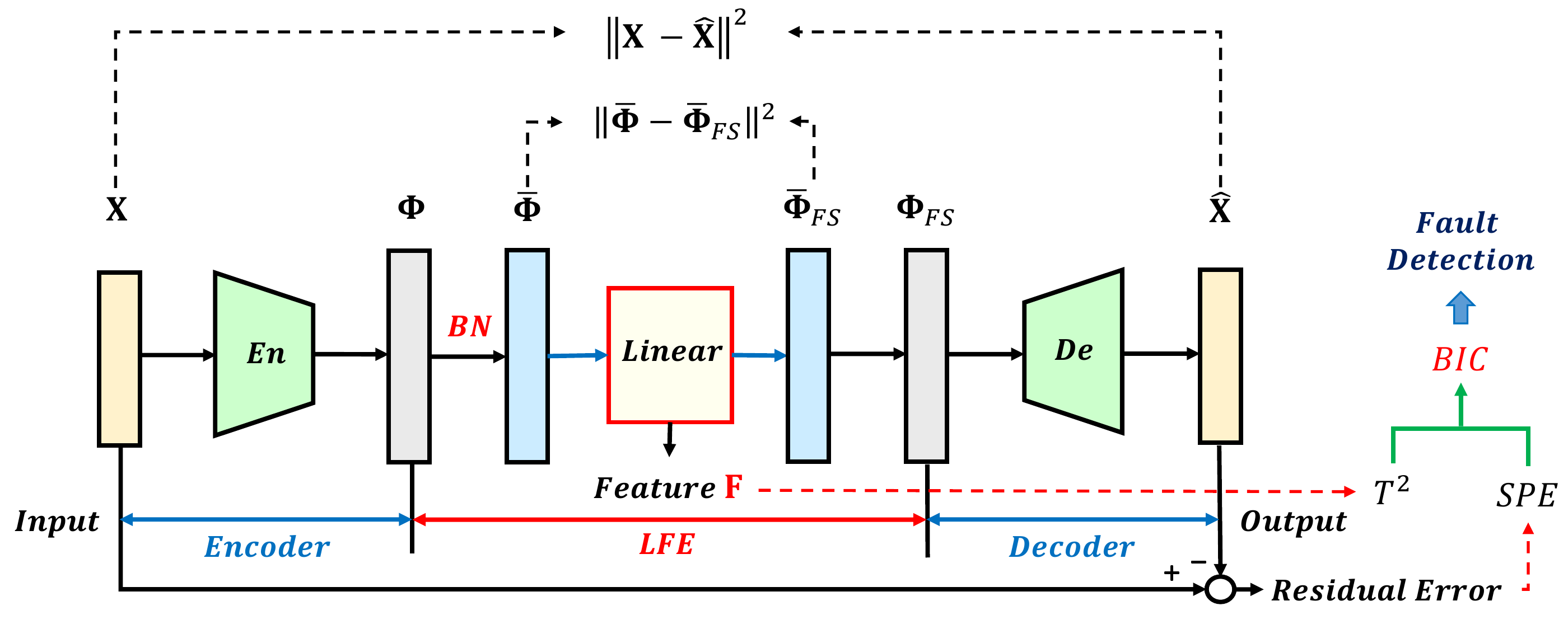}
	\caption{The whole network structure of DAE-FE framework.}
	\label{fig:DAE-FE}
\end{figure}

\textbf{\emph{Remark 1:}} Definition 1 gives the definition of the kernel function. The sufficient condition from Lemma 1 is a vital criterion to prove whether a function can realize to construct a corresponding kernel function. When Lemma 1 is satisfied, Lemma 2 tells that a RHKS associated with $K$ must be exist.

The above preliminaries provide the theoretical basis to further prove why the function ${En\left(  \cdot  \right)}$ from the encoder module is equivalent to a learnable function whose inner product is a kernel function. Now, Proposition 1 is advanced as:

\begin{proposition}The function $En\left(  \cdot  \right)$ from the encoder module is a nonlinear mapping whose inner product is a learnable kernel function. In addition, the space $ \mathcal{H}$ associated with this learnable kernel is a RKHS.\end{proposition}

\begin{proof} In the encoder module, the nonlinear activation functions are real-valued continuous functions and it only contain linear operators except for activation functions. Thus, as their composite functions, ${En\left(  \cdot  \right)}$ is a real-valued continuous function. According to Definition 1, a continuous and symmetric kernel $K$ can be expressed by
\begin{equation}
	K = \left\langle {En\left(  \cdot  \right),En\left(  \cdot  \right)} \right\rangle  = En{\left(  \cdot  \right)^{\top}}En\left(  \cdot  \right).
\end{equation}
The input $\mathbf{X}$ is a matrix defined in the real domain. $\mathbf{\Phi}  = En\left(\mathbf{X} \right)$ is also in the real domain. Thus, $\mathbf{K}$ is the Gram matrix with respect to $K$ and $\mathbf{X}$ can be calculated by
\begin{equation}
	\mathbf{K} = \mathbf{\Phi} {\mathbf{\Phi} ^{\top}}.
\label{eq:K}
\end{equation}
From (\ref{eq:K}), $\mathbf{K}$ is a real symmetric matrix. For any non-zero vector $\mathbf{x}\in {\mathbb{R}^N}$, it holds that
\begin{equation}
{\mathbf{x}^{\top}}\mathbf{K}\mathbf{x} = {\mathbf{x}^{\top}}\mathbf{\Phi} {\mathbf{\Phi} ^{\top}}\mathbf{x} = {\left( {{\mathbf{\Phi} ^{\top}}\mathbf{x}} \right)^{\top}}\left( {{\mathbf{\Phi} ^{\top}}\mathbf{x}} \right) \ge \mathbf{0}.
\end{equation}
Thus, $\mathbf{K}$ is SPSD. According to Lemma 1, the kernel $K$ is PDS kernel. Owing to Lemma 2, since $K$ is a PSD kernel, the space $\mathcal{H}$ is a Hilbert space, and more specifically, a RKHS.\end{proof}

From Proposition 1, it can be seen that $En\left(  \cdot  \right)$ is in essence a nonlinear mapping function, of which the inner product is a learnable kernel function. Different neural network structures and parameters will correspond to a different $En\left(  \cdot  \right)$, and thus to a different kernel function. During training, the neural network executes back-propagation through loss function. When the training is over, a kernel function will be determined and can meet the requirements of the task.

\subsubsection{LFE Module}
The major role of this module is to achieve linear feature extraction or dimensionality reduction for $\mathbf{\Phi} $.

In general, the input of LFE methods (such as PCA) needs to be standardized in each dimension, that is, the zero-mean and unit variance. Its main purpose is as follows: on one hand, the data of each dimension can be transformed to the same scale and pulled to the same baseline (origin of coordinates), so as to guarantee that each dimension has equal importance. On another hand, it is beneficial to the convergence of the gradient descent method.

To standardize $\mathbf{\Phi} $ in $\mathcal{H}$ that is as the output of the encoder module, we introduce the Batch Normalization (BN) trick \cite{ioffe2015icml} into the encoder module. Two trainable parameters in BN trick require to be fixed as $\gamma  = 1$ and $\beta  = 0$. In specific, $\mathbf{\bar \Phi} $ can be expressed as
\begin{equation}
\mathbf{\bar \Phi}  = BN\left( \mathbf{\Phi}  \right) = \frac{{\mathbf{\Phi}  - \mu \left( \mathbf{\Phi}  \right)}}{{\sigma \left( \mathbf{\Phi}  \right)}},
\end{equation}
where $\mu \left( \mathbf{\Phi}  \right)$ and $\sigma \left( \mathbf{\Phi}  \right)$ are the mean and standard deviation of $\mathbf{\Phi}$, respectively. Thus,
\begin{equation}
	\frac{1}{N}\mathbf{1}_N^{\top}\mathbf{\bar \Phi}  = \mathbf{0}.
\label{eq:phimean}
\end{equation}
It should be noted that the kernel method just executes centralization for $\mathbf{\Phi} $, because $\mathbf{\Phi} $ in (\ref{eq:He}) is unknown, its variance cannot be processed to 1. Compared with the centralization in kernel method for $\mathbf{\Phi} $, our standardized treatment has better performance in data preprocessing.

The following network structure is designed in accordance with the given LFE method. The linear dimensionality reduction of $\mathbf{\bar \Phi} $ is first performed to obtain the system feature $\mathbf{F}$. Then, $\mathbf{\bar \Phi}$ in $\mathcal{H}$ is reconstructed as $\mathbf{\bar \Phi} _{FS} $ by $\mathbf{F}$ under the constraint of loss function $los{s_\mathbf{\Phi} }$, which is expressed as:
\begin{equation}
	los{s_\mathbf{\Phi} } = \left\| {\mathbf{\bar \Phi}  - \mathbf{\bar \Phi}_{FS} } \right\|_F^2 + {\lambda _\mathbf{\Phi} }{\Omega _\mathbf{ \Phi} },
\end{equation}
where ${\lambda _\mathbf{\Phi} }$ is a non-negative hyper-parameter, and ${\Omega _\mathbf{\Phi} }$ is the regularization term in $ los{s_\mathbf{\Phi} }$.

Further, $\mathbf{\bar \Phi}$ is decomposed into the following parts:
\begin{equation}
\mathbf{\bar \Phi}  = \mathbf{\bar \Phi}_{FS}  + \mathbf{\bar \Phi}_{RS},
\end{equation}
\begin{equation}
\mathbf{\bar \Phi}_{FS}  = {f_1}\left(\mathbf{F} \right),
\end{equation}
\begin{equation}
\mathbf{F} = {f_2}\left( {\mathbf{\bar \Phi} } \right),
\end{equation}
where $\mathbf{\bar \Phi}_{RS} $ denotes the residual part of $\mathbf{\bar \Phi} $, and ${f_1}\left(  \cdot  \right)$ and ${f_2}\left(  \cdot  \right)$ are both linear functions. The neural network from $\mathbf{\bar \Phi} _{FS} $ to $\mathbf{\Phi} _{FS} $ is a fully connected layer that is viewed as the inverse process of BN trick.

\subsubsection{Decoder Module}
It is designed for the purpose of mapping the output $\mathbf{\Phi} _{FS} $ of LFE module from $\mathcal{H}$ back to $\mathcal{X}$, and reconstruct the original input $\mathbf{X}$ under the constraint of the loss function $los{s_{\mathbf{X}}}$ as
\begin{equation}
	los{s_{\mathbf{X}}} = \left\| {\mathbf{X} - \mathbf{\hat X}} \right\|_F^2.
\end{equation}

Since the decoder network structure is almost symmetric with the encoder, it can be regarded as the inverse process of the encoder and effectively realizes the self-supervised learning for $\mathbf{X}$.

Through the above descriptions of DAE-FE framework, the entire loss function can be expressed as
\begin{equation}
\begin{array}{l}
	loss = {\lambda _1}los{s_{\mathbf{X}}} + {\lambda _2}los{s_\mathbf{\Phi} } + {\lambda _3}\Omega,
\end{array}
\label{eq:loss}
\end{equation}
where ${\lambda _i}\left( {i = 1,2,3} \right)$ is a non-negative hyper-parameter, and $\Omega $ is the regularization term for a specific to a fault detection task. According to (\ref{eq:loss}), all three modules of DAE-FE are jointly trained through backpropagation. After training DAE-FE network, we can use $\mathbf{F}$ to monitor $\mathcal{FS}$ and $\mathbf{\tilde X}$ to monitor $\mathcal{RS}$.

It is worth noticing that as the LFE module can be designed according to a specific LFE method, we can leverage DAE-FE as a learnable kernel method to convert any LFE method to the nonlinear version.

\subsection{DAE-PCA Method}
In order to achieve the function of PCA, the matrix $\mathbf{P}$ learned by the neural network must be orthogonal. In other words, hard constraint is necessary to guarantee the orthogonality during the training process. We propose to solve this issue with the aid of Cayley Transform. Moreover, in order to improve the sensitivity of the test statistics based on the feature variances, we provide a new regularization term to restrict the variances of the features (in the fault-free condition). Our DAE-PCA method is illustrated in Fig. \ref{fig:DAE-PCA}.

\subsubsection{PCA Module}
For this module, the optimization problem can be described like (\ref{eq:KPCA}) of KPCA, i.e.,
\begin{equation}
	\begin{array}{l}
		\mathop {\min }\limits_\mathbf{P} \left\| {\mathbf{\bar \Phi}  - \mathbf{\bar \Phi}\mathbf{P}{\mathbf{P}^{\top}}} \right\|_F^2\\
		s.t.{\kern 1pt} {\kern 1pt} {\kern 1pt} {\mathbf{P}^{\top}}\mathbf{P} = {\mathbf{I}_a}.
	\end{array}
\end{equation}

Note that the input of the PCA module is obtained by the encoder module, rather than kernel trick.

\begin{figure*}[!t]
	\setlength{\abovecaptionskip}{0pt}
	\setlength{\belowcaptionskip}{0pt}
	\centering
	\includegraphics[width=0.6\textwidth]{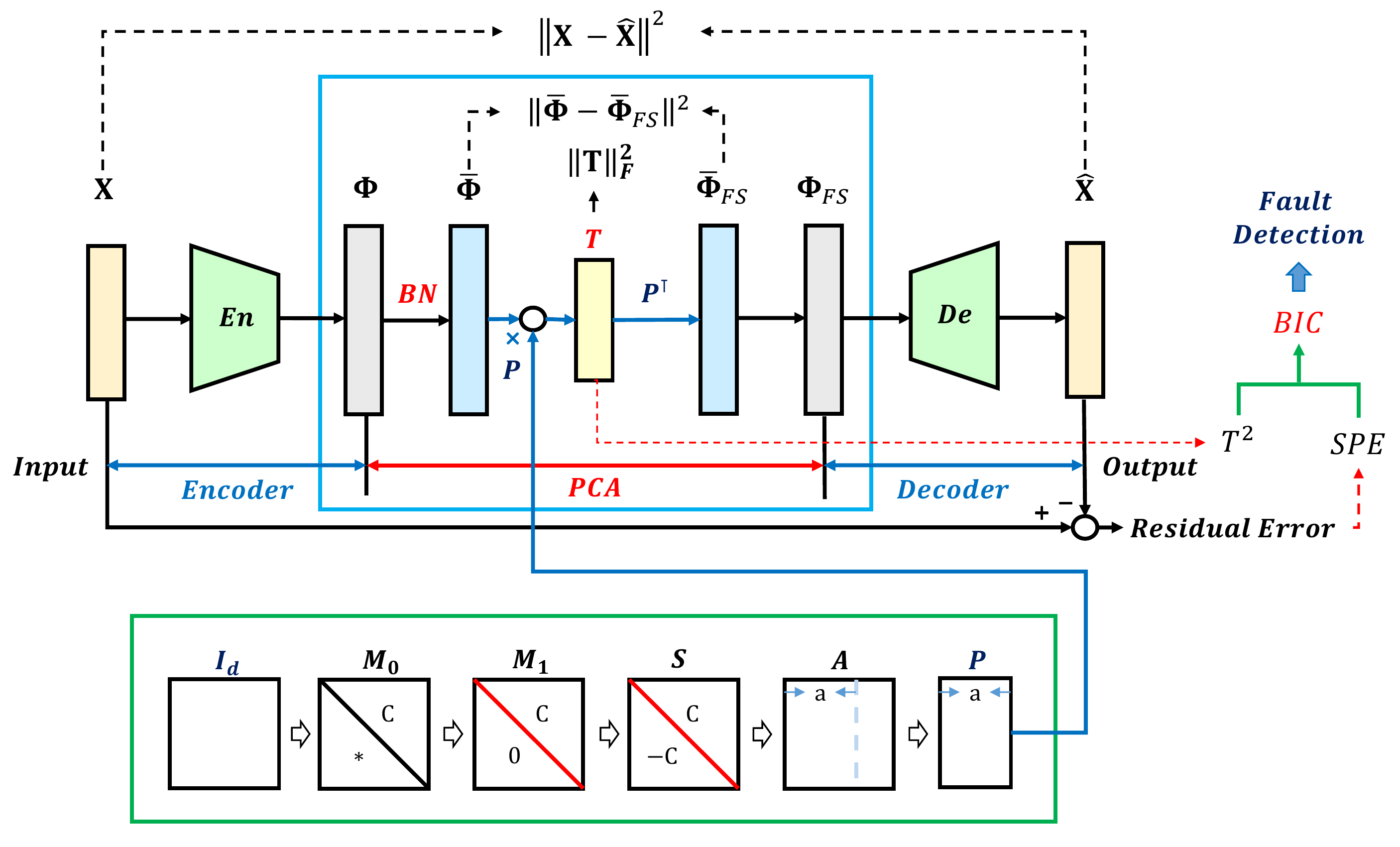}
	\caption{The whole network structure of DAE-PCA, where the PCA module is shown with a solid blue border and the acquisition of $\mathbf{P}$ is in a solid green border.}
	\label{fig:DAE-PCA}
\end{figure*}

Let the loss function $los{s_{\mathbf{\Phi} \_\rm{PCA}}}$ of PCA module be $los{s_\mathbf{\Phi} }$ without ${\Omega _\mathbf{\Phi} }$:
\begin{equation}
	los{s_{\mathbf{\Phi} \_\rm{PCA}}} = \left\| {\mathbf{\bar \Phi}  - {{{\mathbf{\bar \Phi} }_{FS}}} } \right\|_F^2 = \left\| {\mathbf{\bar \Phi}  - \mathbf{\bar \Phi} \mathbf{P}{\mathbf{P}^{\top}}} \right\|_F^2.
\end{equation}
Thus, in DAE-PCA, training the network to minimize $los{s_{\mathbf{\Phi} \_\rm{PCA}}}$ is equivalent to $\mathop {\min }\limits_\mathbf{P} \left\| {\mathbf{\bar \Phi}  - \mathbf{\bar \Phi}  \mathbf{P}{\mathbf{P}^{\top}}} \right\|_F^2$. Thus, the system feature $\mathbf{T}$ is obtained by
\begin{equation}
\mathbf{T} = \mathbf{\bar \Phi}\mathbf{ P}.
\label{eq:T}
\end{equation}

At the same time, we need to restrict $\mathbf{P}$ to satisfy the constraint that ${\mathbf{P}^{\top}}\mathbf{P} = {\mathbf{I}_a}$. For such purpose, we devise the neural network based on Cayley Transform to acquire $\mathbf{P}$. Cayley Transform was first established by Arthur Cayley, one proposition of which is given as follows \cite{cayley1846}.
\begin{proposition}If $\mathbf{S}$ is a real antisymmetric matrix and $\mathbf{I}$ is the identity matrix, then $\mathbf{A}=\left( {\mathbf{I} - \mathbf{S}} \right){\left( {\mathbf{I} + \mathbf{S}} \right)^{ - 1}}$ is an orthonormal matrix.	\end{proposition}

$\mathbf{P}$ can be directly obtained by the network structure as shown in the solid green border in Fig. \ref{fig:DAE-PCA}, of which the input is fixed to the identity matrix ${I_{d \times d}}$ and the output is an orthogonal matrix $\mathbf{P}$ that is acquired by the neural networks restricted based on Cayley Transform. Through training the entire DAE-PCA network with back-propagation mechanism, an optimal $P$ suitable for the PCA module can be learned. The detailed descriptions for acquiring $\mathbf{P}$ in solid green border are in the following steps:
\begin{enumerate}
\item[1)] Fix the input to be identity matrix $\mathbf{I} \in {\mathbb{R}^{d \times d}}$ and initialize a square matrix ${\mathbf{M}_0} \in {\mathbb{R}^{d \times d}}$. Thus, ${\mathbf{M}_0} \in {\mathbb{R}^{d \times d}}$ is obtained by ${\mathbf{M}_0} = {\mathbf{M}_0}\mathbf{I}$;
	
\item[2)] Get the upper triangular matrix ${\mathbf{M}_1} \in {\mathbb{R}^{d \times d}}$ of ${\mathbf{M}_0}$;
	
\item[3)] Obtain the matrix $\mathbf{S} \in {\mathbb{R}^{d \times d}}$ by $\mathbf{S} = {\mathbf{M}_1} - {\mathbf{M}_{1}^{\top}}$;
	
\item[4)] According to Proposition 2, obtain an orthogonal square $\mathbf{A} \in {\mathbb{R}^{d \times d}}$ by $\mathbf{A} = \left( {\mathbf{I} - \mathbf{S}} \right){\left( {\mathbf{I} + \mathbf{S}} \right)^{ - 1}}$;

\item[5)] Take the first $a$ orthogonal column vectors of $\mathbf{A}$ to get the orthogonal projection matrix $\mathbf{P} \in {\mathbb{R}^{d \times a}}$, where $a$ is similar to the principle component number in PCA;

\item[6)] The most optimal projection matrix $\mathbf{P}$ for the PCA module will be acquired by back-propagation mechanism.
	
\end{enumerate}
%
%
%
%
Notice that learning for matrix $\mathbf{P}$ is in essence to learn about matrix $\mathbf{M}_0$. Once $\mathbf{M}_0$ is determined, $\mathbf{P}$ is also determined.

The above trick is implemented by a hard constraint, rather than a soft constraint by adding any regularization term about orthogonality, i.e., ${\Omega _\mathbf{\Phi}} {\rm{ = }}{\left\| {\mathbf{P}^{\top}}\mathbf{P} - {\mathbf{I}_a} \right\|_F^2}$. Compared with the soft constraint, the advantage of our hard constraint is that its orthogonal performance is always maintained and is not be affected by the design of the loss function.

Thus, the function $los{s_{\rm {DAE-PCA, 1}}}$ our DAE-PCA is expressed as
\begin{equation}
\begin{aligned}
	los{s_{\rm {DAE-PCA, 1}}} = {\lambda _1}los{s_ \mathbf{ X}} + {\lambda _2}los{s_{ \mathbf{ \Phi} \_{\rm{PCA}}}}\\
	={\lambda _1}{\left\| {\mathbf{X} - \mathbf{\hat X}} \right\|_F^2}+{\lambda _2}{\left\| {\mathbf{\bar \Phi}  - {{{\mathbf{\bar \Phi} }_{FS}}} } \right\|_F^2}.
\end{aligned}
\end{equation}
In addition, set ${\lambda _1} = \frac{1}{{N \times m}}$ and ${\lambda _2} = \frac{1}{{N \times d}}$ that aims to eliminate the effect of the number of matrix elements for each item of loss.

\begin{figure}[!t]
	\setlength{\abovecaptionskip}{0pt}
	\setlength{\belowcaptionskip}{0pt}
	\centering
	\includegraphics[width=0.98\columnwidth]{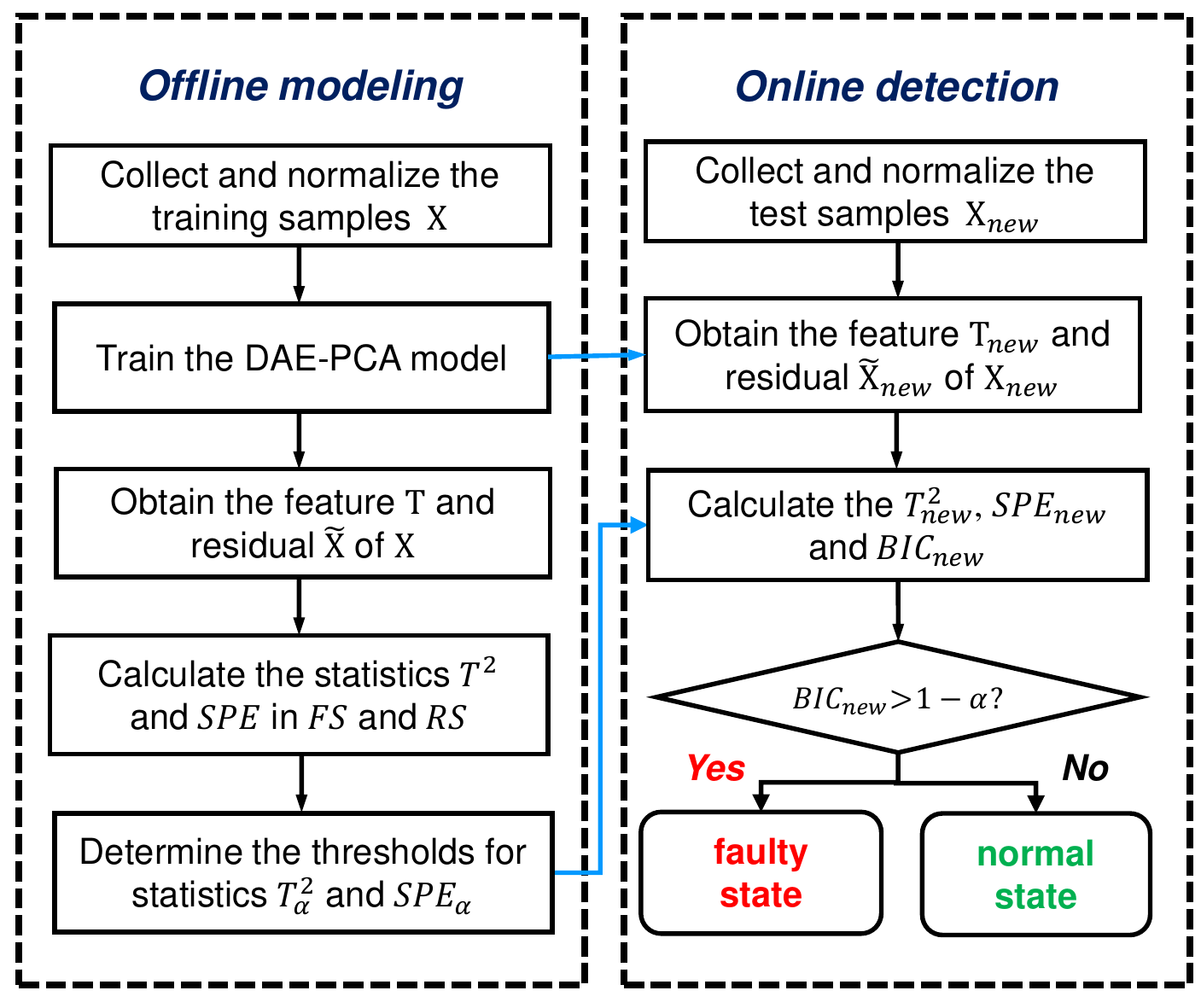}
	\caption{The flowchart of nonlinear fault detection based on DAE-PCA.}
	\label{fig:logic}
\end{figure}

\subsubsection{Regularization Term ${\Omega _\mathbf{T}}$} Set ${\mathbf{s}^i}\in {\mathbb{R}^N}$ to be the $i_{th}$ column vector of $\mathbf{T}$, and then $\mathbf{T} = \left[ {{\mathbf{s}^1},{\mathbf{s}^2}, \ldots ,{\mathbf{s}^a}} \right]$. Let $\mu \left( {{\mathbf{s}^i}} \right)$ and ${\sigma ^2} \left( {{\mathbf{s}^i}} \right)$ denote the mean and variance of $\mathbf{s}^i$.

According to (\ref{eq:phimean}) and (\ref{eq:T}), it holds that
\begin{equation}
\mu \left( \mathbf{T} \right)= \left[ {\mu \left( {{\mathbf{s}^1}} \right),\mu \left( {{\mathbf{s}^2}} \right), \ldots ,\mu \left( {{\mathbf{s}^a}} \right)} \right] \in {\mathbb{R}^{1 \times a}} = \mathbf{0},
\end{equation}
because
\begin{equation}
\mu \left( \mathbf{T} \right) = \frac{1}{N}\mathbf{1}_N^{\top}\mathbf{T} =\frac{1}{N}\mathbf{1}_N^{\top}\mathbf{\bar \Phi} \mathbf{P} = \left({\frac{1}{N}\mathbf{1}_N^{\top}\mathbf{\bar \Phi}} \right)\mathbf{P} = \mathbf{0}.
\end{equation}
Further, we have
\begin{equation}
{\sigma ^2}\left( {{\mathbf{s}^i}} \right) = \frac{1}{N}\sum\limits_{j = 1}^N {{{\left( {\mathbf{s}_j^i - \mu \left( {{\mathbf{s}^i}} \right)} \right)}^2}}  = \frac{1}{N}\sum\limits_{j = 1}^N {{{\left( {\mathbf{s}_j^i} \right)}^2}},
\end{equation}
\begin{equation}
\sum\limits_{i = 1}^a {{\sigma ^2}\left( {{\mathbf{s}^i}} \right)}  = \frac{1}{N}\sum\limits_{i = 1}^a {\sum\limits_{j = 1}^N {{{\left( {\mathbf{s}_j^i} \right)}^2}} }  = \frac{1}{N}\left\| \mathbf{T} \right\|_F^2,
\label{eq:sumvariance}
\end{equation}
where ${\mathbf{s}_j^i}$ represents the element of $\mathbf{T}$ corresponding to $i_{th}$ sensor and $j_{th}$ sample point.

Based on (\ref{eq:sumvariance}), we define a regularization item ${\Omega _\mathbf{T}}$:
\begin{equation}
{\Omega _\mathbf{T}} = \left\| \mathbf{T} \right\|_F^2 \propto \sum\limits_{i = 1}^a {{\sigma ^2}\left( {{\mathbf{s}^i}} \right)},
\end{equation}
which represents the sum of variance of $\mathbf{T}$ on all feature dimensions. The decrease of $\Omega _\mathbf{T}$ helps the system feature $\mathbf{T}$ to be more compact in space, so as to better distinguish the faulty samples from the normal samples.

Through the above analysis, the function $los{s_{\rm {DAE-PCA, 2}}}$ of our DAE-PCA with ${\Omega _\mathbf{T}}$ is defined as
\begin{equation}
\begin{aligned}
	los{s_{\rm {DAE-PCA, 2}}} = los{s_{\rm {DAE-PCA, 1}}}+{\lambda _3}{\Omega _\mathbf{T}}\\
	=los{s_{\rm {DAE-PCA, 1}}}+{\lambda _3}{\left\| \mathbf{T} \right\|_F^2}.
\end{aligned}
\end{equation}

\begin{figure*}[b]
{\noindent}	 \rule[-3pt]{17.5cm}{0.05em}
\begin{equation}
BIC\left( {{\mathbf{x}_{new}}} \right) = \frac{{{P_{{T^2}}}\left( {{\mathbf{x}_{new}}\left| F \right.} \right){P_{{T^2}}}\left( {F\left| {{\mathbf{x}_{new}}} \right.} \right) + {P_{SPE}}\left( {{\mathbf{x}_{new}}\left| F \right.} \right){P_{SPE}}\left( {F\left| {{\mathbf{x}_{new}}} \right.} \right)}}{{{P_{{T^2}}}\left( {{\mathbf{x}_{new}}\left| F \right.} \right) + {P_{SPE}}\left( {{\mathbf{x}_{new}}\left| F \right.} \right)}}.
\label{eq:BIC}
\end{equation}
\end{figure*}

\subsection{DAE-PCA Based Nonlinear Fault Detection}
To detect whether some faults occur in the system adopting DAE-PCA method, we employ the two most commonly used statistics in fault diagnosis community to monitor $\mathcal{PS}$ and $\mathcal{RS}$ respectively, i.e., $ {\rm{Hotelling's}}{\kern 1pt} {\kern 1pt} {\kern 1pt}{T^2}$ statistics and $\rm{SPE}$ statistics.

The flowchart of nonlinear fault detection based on DAE-PCA is summarized in Fig. \ref{fig:logic}. Given the training data matrix $ \mathbf{ X }\in { \mathbb{ R}^{N \times m}}$, according to (\ref{eq:T}) and (\ref{eq:xrs}), we have the system feature matrix $ \mathbf{ T }\in { \mathbb{ R}^{N \times a}}$ and residual matrix $ \mathbf{ \tilde X }\in { \mathbb{ R}^{N \times m}}$. Set $\mathbf{x} \in {\mathbb{R}^m}$ to be a sample of $ \mathbf{ X }$, its Hotelling's $ {T^2}$ statistics and squared prediction error ($SPE$) statistics can be defined as follows:
\begin{equation}
	{T^2}\left(\mathbf{x} \right) = {\mathbf{t}^{\top}} {\mathbf{\Lambda} }^{ - 1}\mathbf{t},
\end{equation}
\begin{equation}
	{SPE}\left(\mathbf{x} \right) = {\mathbf{\tilde x}^{\top}}\mathbf{\tilde x},
\end{equation}
where $\mathbf{t} \in {\mathbb{R}^a}$, $\mathbf{\tilde x} \in {\mathbb{R}^m}$ are the system feature vector and residual vector with respect to $\mathbf{x}$, and $\mathbf{\Lambda}  =  {\mathbf{T}{\mathbf{T}^{\top}}/\left( {N - 1} \right)}$ is the covariance matrix of $\mathbf{t}$.

We set the threshold of $ {T^2}$ statistics and $SPE$ statistics as ${J_{th,{T^2}}}$ and ${J_{th, {SPE}}}$ respectively, which are calculated by applying ${T^2}$ and ${\rm{SPE}}$ to the kernel probability estimation (KDE) scheme under the given confidence limit $\alpha $. The calculation method about KDE can be referred to \cite{parzen1962}.

For a test sample $\mathbf{x}_{new} \in {\mathbb{R}^m}$, let its system feature and residual be $\mathbf{t}_{new} \in {\mathbb{R}^m}$ and $\mathbf{\tilde x}_{new} \in {\mathbb{R}^m}$, respectively. Thus, $ {T^2}$ statistics and $\rm{SPE}$ statistics of $\mathbf{x}_{new}$ are computed by
\begin{equation}
	{T^2}\left(\mathbf{x}_{new} \right) = {\mathbf{t}_{new}^{\top}} {\mathbf{\Lambda} }^{ - 1}\mathbf{t}_{new},
\end{equation}
\begin{equation}
	{SPE}\left(\mathbf{x}_{new} \right) = {\mathbf{\tilde x}_{new}^{\top}}\mathbf{\tilde x}_{new}.
\end{equation}

In order to provide a clear detection logic, the statistics ${T^2}$ and $SPE$ are integrated into a statistics, Bayesian information criterion ($BIC$), according to the Bayesian inference \cite{ge2010jpc, huang2019tcst}. To simplify the descriptions, these two statistics and their thresholds are represented by two sets, $S{\rm{ = }}\left\{ {{T^2},SPE} \right\}$ and ${J_{th,S}} = \left\{ {{J_{th,{T^2},}}{J_{th,SPE}}} \right\}$ respectively, since the following derivations are the same for them.

The fault posterior probability of $S$ is calculated by
\begin{equation}
{P_S}\left( {F\left| \mathbf{x}_{new} \right.} \right) = \frac{{{P_S}\left( {\mathbf{x}_{new}\left| F \right.} \right){P_S}\left( F \right)}}{{{P_S}\left( \mathbf{x}_{new} \right)}},
\end{equation}
in which ${P_S}\left( \mathbf{x}_{new} \right)$ is computed by
\begin{equation}
{P_S}\left( {{\mathbf{x}_{new}}} \right) = {P_S}\left( {{\mathbf{x}_{new}}\left| N \right.} \right){P_S}\left( N \right) + {P_S}\left( {{\mathbf{x}_{new}}\left| F \right.} \right){P_S}\left( F \right)
\end{equation}
where $N$ and $F$ denote the normal and faulty state. The normal prior probability ${{P_S}\left( N \right)}$ equals to $\alpha $ and the faulty prior probability ${{P_S}\left( F \right)}$ is $1 - \alpha$. The normal and faulty conditional probability are defined as
\begin{equation}
\begin{array}{l}
{P_S}\left( {{{\mathbf{x}_{new}}}\left| N \right.} \right) = \exp \left\{ { - S/{J_{th,S}}} \right\}\\
{P_S}\left( {{{\mathbf{x}_{new}}}\left| F \right.} \right) = \exp \left\{ { - {J_{th,S}}/S} \right\}.
\end{array}
\end{equation}

Further, the statistics $BIC$ can be obtained in (\ref{eq:BIC}), which is adopted to monitor the full space that is called $\mathcal{FS}$. $\mathcal{FS}$ can be regarded as the space by combing $\mathcal{PS}$ with $\mathcal{RS}$. Thus, the detection logic in $\mathcal{FS}$ is expressed as follows:

\begin{equation}
\left\{ {\begin{array}{*{20}{l}}
{BIC\left( {{\mathbf{x}_{new}}} \right) \le 1 - \alpha  \Rightarrow {x_{new}}{\kern 1pt} {\kern 1pt} {\kern 1pt} {\rm is{\kern 1pt} {\kern 1pt} {\kern 1pt} a {\kern 1pt} {\kern 1pt} {\kern 1pt} fault-free{\kern 1pt} {\kern 1pt} {\kern 1pt} sample}}\\
{BIC\left( {{\mathbf{x}_{new}}} \right) > 1 - \alpha  \Rightarrow {x_{new}}{\kern 1pt} {\kern 1pt} {\kern 1pt} {\rm is{\kern 1pt} {\kern 1pt} {\kern 1pt} a {\kern 1pt} {\kern 1pt} {\kern 1pt} faulty{\kern 1pt} {\kern 1pt} {\kern 1pt} sample}}
\end{array}} \right.
\label{eq:dl}
\end{equation}

\section{Case Study}
To show the performance and the superiority of the proposed DAE-PCA, the TE process is employed in our experiment. In this section, the proposed method will be compared with several comparative methods to comprehensively validate the  performance in fault detection tasks.

\subsection{Simulation Setup}
\subsubsection{Benchmark}
The TE process is a real industrial benchmark which has been widely applied for the simulation and verification of process monitoring methods \cite{chiang2000}, and its diagram is
displayed in Fig. \ref{fig:TEP}. The TE process consists of two variable blocks: one is the XMV block composed by 11 manipulated variables and another is the XMEAS block that comprises 41 measured variables including 22 process variables and 19 analysis variables. Further information of the benchmark can be referred to \cite{downs1993} and the website \footnote{http://depts.washington.edu/control/LARRY/TE/download.html}.

\begin{figure}[!bp]
	\setlength{\abovecaptionskip}{0pt}
	\setlength{\belowcaptionskip}{0pt}
	\centering
	\includegraphics[width=0.92\columnwidth]{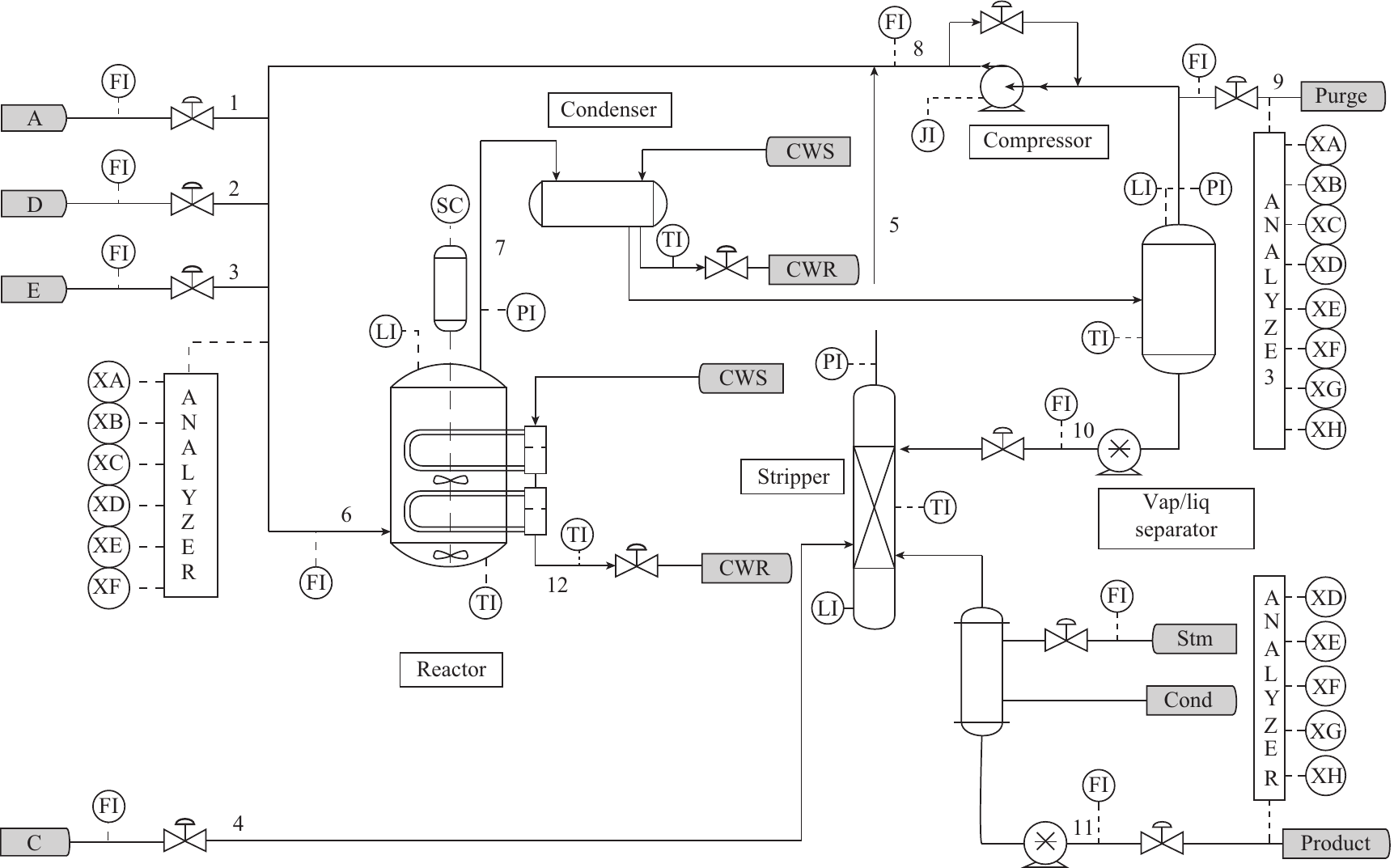}
	\caption{The diagram of the TE process.}
	\label{fig:TEP}
\end{figure}

\begin{table}[!bp]
\centering
\caption{Fault types in the TE process.}
\scalebox{0.85}{
\begin{tabular}{lll}
\toprule
Fault& Description& Type\\
\midrule
IDV(1)& A/C Feedratio, B composition constant(Stream 4)& Step\\
IDV(2)& B composition, A/C ratio constant(Stream 4)& Step\\
IDV(3)& D feed temperature (Stream 2)& Step\\
IDV(4)& Reactor cooling water inlet temperature& Step\\
IDV(5)& Condenser cooling water inlet temperature& Step\\
IDV(6)& A feed loss (Stream 1)& Step\\
IDV(7)& C header pressure loss (Stream 4)& Step\\
IDV(8)& A, B, C feed composition (Stream 4)& Random variation\\
IDV(9)& D feed temperature (Stream 2)& Random variation\\
IDV(10)& C feed temperature (Stream 4)& Random variation\\
IDV(11)& Reactor cooling water inlet temperature& Random variation\\
IDV(12)& Condenser cooling water inlet temperature& Random variation\\
IDV(13)& Reactor kinetics& Slow drift\\
IDV(14)& Reactor cooling water valve& Sticking\\
IDV(15)& Condenser cooling water valve& Sticking\\
IDV(16)& Unknown&  Unknown\\
IDV(17)& Unknown&  Unknown\\
IDV(18)& Unknown&  Unknown\\
IDV(19)& Unknown&  Unknown\\
IDV(20)& Unknown&  Unknown\\
IDV(21)& Valve (Stream 4)& Constant position\\
\bottomrule
\end{tabular}}
\label{table:faulttypes}
\end{table}

\renewcommand{\arraystretch}{1.09}
\begin{table*}[!htbp]
\centering
\fontsize{6}{7.5}\selectfont
\begin{threeparttable}
\caption{FDRs (\%) of two subspaces $ \mathcal{PS}$ and $ \mathcal{RS}$ as well as the full space $ \mathcal{FS}$ by KPCA, DAE, DAE-PCA-1 and DAE-PCA-2 in the TE process. In addition, the value in bold represents the best performance under the same fault and statistics for all methods. Avg. gives the average of the mean of statistics for the first two categories of faults.}\label{table:fdrresults}
\begin{tabular}{llllllllllllllll}
\toprule
\multicolumn{1}{c}{Fault}&\multicolumn{3}{c}{KPCA}&\multicolumn{3}{c}{DAE}&\multicolumn{3}{c}{DAE-PCA-1}&\multicolumn{3}{c}{DAE-PCA-2}\\
 \cmidrule(ll){2-4} \cmidrule(ll){5-7}\cmidrule(ll){8-10}\cmidrule(ll){11-13}
 No.   & $ \mathcal{PS}$& $ \mathcal{RS}$& $ \mathcal{FS}$& $ \mathcal{PS}$& $ \mathcal{RS}$& $ \mathcal{FS}$& $ \mathcal{PS}$& $ \mathcal{RS}$& $ \mathcal{FS}$& $ \mathcal{PS}$& $ \mathcal{RS}$& $ \mathcal{FS}$\cr
\midrule
1 &   100& 99.63& 99.88& 99.99 $\pm$ 0.06& 99.53 $\pm$ 0.20&  99.95$\pm$ 0.08&   100 $\pm$    0&    100$\pm$    0& 99.97 $\pm$ 0.06& 99.96 $\pm$ 0.06& 99.82 $\pm$ 0.08&  99.88$\pm$ 0.09\\
2 & 98.63& 98.50& 98.62 & 98.68 $\pm$ 0.12& 98.17 $\pm$ 0.36&  98.71$\pm$ 0.12& 98.65 $\pm$ 0.06&  98.76$\pm$ 0.09& 98.73 $\pm$ 0.09& 98.67 $\pm$ 0.11& 97.48 $\pm$ 0.66&  98.54$\pm$ 0.1 \\
6 &   100&   100&   100&   100 $\pm$    0&   100 $\pm$    0&    100$\pm$    0&   100 $\pm$    0&    100$\pm$    0&   100 $\pm$    0&   100 $\pm$    0&   100 $\pm$    0&    100$\pm$    0\\
7 &   100&   100&   100&  100  $\pm$    0& 99.82 $\pm$ 0.54&    100$\pm$    0&   100 $\pm$    0&    100$\pm$    0&   100 $\pm$    0&   100 $\pm$    0& 99.76 $\pm$ 1.42&    100$\pm$    0\\
8 & 99.00& 98.75& 99   & 98.38 $\pm$ 0.22& 97.53 $\pm$ 0.34&  98.37$\pm$ 0.21& 98.35 $\pm$ 0.19&  98.46$\pm$ 0.22& 98.36 $\pm$ 0.2 & 98.20 $\pm$ 0.09& 95.79 $\pm$ 1.36&  98.03$\pm$ 0.1 \\
12& 99.63& 99.75& 99.75& 99.62 $\pm$ 0.13& 99.74 $\pm$ 0.13&  99.82$\pm$ 0.08& 99.78 $\pm$ 0.10&  99.88$\pm$ 0.09& 99.85 $\pm$ 0.08& 99.87 $\pm$    0& 99.62 $\pm$ 0.20&  99.87$\pm$    0\\
13& 95.25& 94.88& 95.38& 95.35 $\pm$ 0.19& 94.62 $\pm$ 0.32&  95.39$\pm$ 0.17& 95.28 $\pm$ 0.09&   95.4$\pm$ 0.13& 95.39 $\pm$ 0.12& 95.26 $\pm$ 0.04& 95.21 $\pm$ 0.17&  95.28$\pm$ 0.06\\
14&   100&   100&   100&   100 $\pm$    0& 99.84 $\pm$ 0.27&    100$\pm$    0&   100 $\pm$    0&    100$\pm$    0&   100 $\pm$ 0.06&     0 $\pm$    0& 99.93 $\pm$ 0.08&    100$\pm$    0\\
17& 96.63& 92.50& 96.38 & 96.87 $\pm$ 1.25& 87.23 $\pm$ 3.25&  96.56$\pm$ 1.06& 97.15 $\pm$ 0.14&   97.2$\pm$ 0.16& 96.8  $\pm$ 0.24& 96.97 $\pm$ 0.38& 96.12 $\pm$ 0.62&  96.81$\pm$ 0.26\\
18& 90.75& 90.25& 90.62& 90.64 $\pm$ 0.25& 89.88 $\pm$ 0.31&  90.58$\pm$ 0.25& 90.60 $\pm$ 0.15&  90.99$\pm$ 0.25& 90.68 $\pm$ 0.27& 90.49 $\pm$ 0.18& 90.28 $\pm$ 0.22&  90.37$\pm$ 0.19\\
\midrule
4 &   \textbf{100}& 97.88&   \textbf{100}& 99.52 $\pm$ 3.39& 61.80 $\pm$20.45&99.62$\pm$2.67&   \textbf{100} $\pm$    0& 71.95 $\pm$18.38&\textbf{100} $\pm$    0& 99.98 $\pm$ 0.08&   \textbf{100} $\pm$    0&\textbf{100} $\pm$    0\\
5 & 32.88&   \textbf{100}&   \textbf{100}& 95.57 $\pm$ 9.62& 99.96 $\pm$ 0.06&\textbf{100}$\pm$0& 97.85 $\pm$10.63& 99.97 $\pm$ 0.06&\textbf{100} $\pm$    0&   \textbf{100} $\pm$    0& 99.99 $\pm$ 0.03&\textbf{100} $\pm$    0\\
10& 89.00& 67.12& 88.12& 88.67 $\pm$ 1.95& 78.36 $\pm$ 8.08&88.72$\pm$1.04& 90.02 $\pm$ 1.24& 80.22 $\pm$ 6.42&89.69$\pm$1.3 & \textbf{91.51} $\pm$ 0.54& \textbf{80.53} $\pm$ 4.03&\textbf{90.15} $\pm$0.76\\
11& 83.25& 75.75& 81.5 & 83.50 $\pm$ 2.52& 56.79 $\pm$ 4.04&82.83$\pm$2.22& \textbf{83.55} $\pm$ 0.62& 61.85 $\pm$ 4.10&\textbf{83.34}$\pm$0.68& 82.07 $\pm$ 1.55& 75.56 $\pm$ 4.02&81.73$\pm$0.81\\
16& 93.25& 63.5 & 92.88& 92.56 $\pm$ 1.61& 82.07 $\pm$11.16&92.96$\pm$0.99& 93.55 $\pm$ 0.80& \textbf{82.21} $\pm$11.78&\textbf{93.5} $\pm$0.98& \textbf{94.23} $\pm$ 0.47& 81.09 $\pm$ 3.98&92.3$\pm$0.57\\
19& 91.50& 23.62& 89.5 &91.03 $\pm$ 3.76& 88.63 $\pm$ 2.44&92.88$\pm$0.81& 94.14 $\pm$ 1.53& \textbf{88.64} $\pm$ 2.33&\textbf{93.85} $\pm$0.67& \textbf{94.72} $\pm$ 0.43& 72.72 $\pm$13.47&92.58$\pm$0.52\\
20& 73.38& 79.12& 79.88&79.5  $\pm$ 4.47& 78.02 $\pm$ 1.86&82.01$\pm$2.29& 86.72 $\pm$ 4.03& 78.57 $\pm$ 1.74&86.56$\pm$3.8& \textbf{91.22} $\pm$ 0.31& \textbf{86.84} $\pm$ 5.13&\textbf{91.01} $\pm$0.34\\
21& 58.88& 49.88& 55.88&60.55 $\pm$ 2.47& 40.84 $\pm$ 5.53&59.72$\pm$1.98& 62.19 $\pm$ 1.79& 44.87 $\pm$ 4.89&61.39$\pm$1.69& \textbf{64.10} $\pm$ 1.23& \textbf{56.09} $\pm$ 6.68&\textbf{62.33} $\pm$1.87\\
\midrule
Avg.& 89& 85.06&92.63& 92.8& 86.27&93.23& 93.77&87.63&93.78& \textbf{94.29}& \textbf{90.38}&\textbf{93.83}\\
\midrule
\midrule
3 & 10.88 & 13.75&12.75& 9.88 $\pm$ 1.69& 6.58 $\pm$ 1.52& 10.54$\pm$ 1.66& 8.80 $\pm$ 0.94& 6.98 $\pm$ 1.42& 10.29$\pm$ 1.27& 7.69 $\pm$ 0.62& 3.49 $\pm$ 1.09& 4.91$\pm$ 0.85\\
9 &   6.88&  9.88& 9.12& 7.77 $\pm$ 1.39& 5.38 $\pm$ 1.21&  8.32$\pm$ 1.3 & 6.43 $\pm$ 0.71& 5.72 $\pm$ 1.19&  7.99$\pm$ 1.04& 5.60 $\pm$ 0.45& 2.81 $\pm$ 0.84& 3.99$\pm$ 0.49\\
15&  18.50& 17.88&18   & 15.53$\pm$ 2.36& 6.93 $\pm$ 1.76& 14.64$\pm$ 2.04& 15.48$\pm$ 1.93& 8.38 $\pm$ 1.98& 15.59$\pm$ 2.09& 16.07$\pm$ 1.18& 5.89 $\pm$  2.6&11.19$\pm$ 1.99\\
\bottomrule
\end{tabular}
\end{threeparttable}
\end{table*}

\renewcommand{\arraystretch}{1.09}
\begin{table*}[!htbp]
\centering
\fontsize{6}{7.5}\selectfont
\begin{threeparttable}
\caption{FARs (\%) of two subspaces $ \mathcal{PS}$ and $ \mathcal{RS}$ as well as the full space $ \mathcal{FS}$ by KPCA, DAE, DAE-PCA-1 and DAE-PCA-2 in the TE process. In addition, the value in bold represents the best performance under the same fault and statistics for all methods. Avg. gives the average of the mean of statistics for the first two categories of faults.}\label{table:farresults}
\begin{tabular}{llllllllllllllll}
\toprule
\multicolumn{1}{c}{Fault}&\multicolumn{3}{c}{KPCA}&\multicolumn{3}{c}{DAE}&\multicolumn{3}{c}{DAE-PCA-1}&\multicolumn{3}{c}{DAE-PCA-2}\\
 \cmidrule(ll){2-4} \cmidrule(ll){5-7}\cmidrule(ll){8-10}\cmidrule(ll){11-13}
 No.   & $ \mathcal{PS}$& $ \mathcal{RS}$& $ \mathcal{FS}$& $ \mathcal{PS}$& $ \mathcal{RS}$& $ \mathcal{FS}$& $ \mathcal{PS}$& $ \mathcal{RS}$& $ \mathcal{FS}$& $ \mathcal{PS}$& $ \mathcal{RS}$& $ \mathcal{FS}$\cr
\midrule
1 & 0.63& 1.25& 1.88& 1.98 $\pm$ 1.08& 2.06 $\pm$ 1.24&2.57	$\pm$1.58& 1.31 $\pm$ 0.67& 2.4  $\pm$ 1.22&2.68$\pm$1.18& 0.96 $\pm$ 0.36& 1.71 $\pm$ 1.03&1.37$\pm$0.64\\
2 & 0.63& 1.25& 1.25& 1.21 $\pm$ 0.65& 1.64 $\pm$ 0.86&1.99	$\pm$1.73& 0.96 $\pm$ 0.57& 1.86 $\pm$ 0.91&1.75$\pm$0.86& 0.35 $\pm$ 0.34& 0.87 $\pm$ 0.71&0.39$\pm$0.47\\
6 &    0& 0.63& 0.62& 0.68 $\pm$ 0.78& 2.26 $\pm$ 0.83&1.92	$\pm$1.15& 0.39 $\pm$ 0.53& 2.55 $\pm$ 1.14&2.1 $\pm$1.02& 0.34 $\pm$ 0.38& 1.17 $\pm$ 0.84&0.48$\pm$0.6 \\
7 & 1.88& 0.63& 0.62& 2.35 $\pm$ 1.36& 2.93 $\pm$ 1.36&3.5	$\pm$1.75& 2.04 $\pm$ 0.8 & 3.04 $\pm$ 1.38&3.14$\pm$1.35& 1.74 $\pm$ 0.29& 1.17 $\pm$ 1.02&1.12$\pm$0.76\\
8 & 4.38& 3.13& 3.75& 4.2  $\pm$ 1.62& 3.64 $\pm$ 1.6 &4.92	$\pm$1.9 & 3.45 $\pm$ 1.12& 3.92 $\pm$ 1.65&4.72$\pm$1.57& 2.50 $\pm$ 0.64& 1.54 $\pm$ 0.97&1.35$\pm$0.84\\
12& 6.88&   10& 8.75& 4.16 $\pm$ 1.56& 2.36 $\pm$ 1.34&3.92	$\pm$1.95& 2.82 $\pm$ 1.27& 3.04 $\pm$ 1.65&3.84$\pm$1.63& 2.11 $\pm$ 0.58& 1.22 $\pm$ 0.89&0.97$\pm$0.86\\
13& 1.25& 1.25& 1.25& 0.89 $\pm$ 0.77& 1.85 $\pm$ 0.94&1.95	$\pm$1.37& 0.3  $\pm$ 0.52& 2.01 $\pm$ 0.9 &1.88$\pm$0.87&    0 $\pm$    0& 0.55 $\pm$ 0.52&0.23$\pm$0.37\\
14& 1.25& 1.88& 1.88& 2.29 $\pm$ 1.33& 2.31 $\pm$ 1.33&2.95	$\pm$1.95& 1.73 $\pm$ 0.73& 2.29 $\pm$ 1.11&2.59$\pm$1.2 & 1.90 $\pm$ 0.74& 1.23 $\pm$ 0.92&0.81$\pm$0.65\\
17& 1.88& 2.5 & 1.25& 3.76 $\pm$ 1.58& 3.14 $\pm$ 1.71&4.14	$\pm$1.82& 3.36 $\pm$ 0.94& 3.45 $\pm$ 1.50&4.28$\pm$1.47& 2.91 $\pm$ 0.69& 2.48 $\pm$ 1.52&1.94$\pm$1   \\
18& 1.25& 2.5 & 1.88& 3.44 $\pm$ 0.98& 2.01 $\pm$ 1.04&3.44	$\pm$1.23& 2.96 $\pm$ 0.58& 2.16 $\pm$ 1.29&3.33$\pm$1.1 & 2.24 $\pm$ 0.54& 1.57 $\pm$ 0.89&1.74$\pm$0.83\\
\midrule
4 & 0.63& 1.88& 1.88& 1.6  $\pm$ 0.78& 1.8  $\pm$ 1.05&2.19 $\pm$1.07& 1.1  $\pm$ 0.67& 1.69 $\pm$ 1.01&1.96 $\pm$0.86& 0.9  $\pm$ 0.34& 1.89 $\pm$ 1.05& 1.24$\pm$0.61\\
5 & 0.63& 1.88& 1.88& 1.6  $\pm$ 0.78& 1.8  $\pm$ 1.05&2.19 $\pm$1.07& 1.1  $\pm$ 0.67& 1.69 $\pm$ 1.01&1.96 $\pm$0.86& 0.9  $\pm$ 0.34& 1.89 $\pm$ 1.05& 1.24$\pm$0.61\\
10& 1.25& 2.5 & 1.88& 1.65 $\pm$ 0.86& 2.63 $\pm$ 1.16&2.66 $\pm$1.82& 1.07 $\pm$ 0.54& 2.87 $\pm$ 1.25&2.28 $\pm$1.24& 1.06 $\pm$ 0.51& 1.2  $\pm$ 0.9 & 0.56$\pm$0.55\\
11& 2.5 & 2.5 &  2.5& 3.01 $\pm$ 0.96& 2.3  $\pm$ 1.1 &3.6  $\pm$1.27& 2.14 $\pm$ 0.81& 2.33 $\pm$ 1.11&3.4  $\pm$1.12& 1.66 $\pm$ 0.57& 1.61 $\pm$ 0.84& 1.35$\pm$0.61\\
16&23.75&32.5 & 30  &15.91 $\pm$ 3.56& 9.2  $\pm$ 2.91&15.7 $\pm$3.47&12.8  $\pm$ 2.76&10.98 $\pm$ 3.14&15.58$\pm$3.16&11.35 $\pm$ 1.36& 2.66 $\pm$ 1.52& 5.1 $\pm$1.54\\
19&    0& 1.25&    0& 1.45 $\pm$ 1.22& 1.57 $\pm$ 1.15&1.61 $\pm$1.21& 0.99 $\pm$ 0.72& 2.1  $\pm$ 1.15&1.84 $\pm$ 1.1& 0.62 $\pm$ 0.38& 1.17 $\pm$ 0.77& 0.66$\pm$0.53\\
20&    0&    0&    0& 0.79 $\pm$ 0.78& 2.12 $\pm$ 1.07&1.65 $\pm$1.35& 0.28 $\pm$ 0.4 & 2.24 $\pm$ 0.96&1.62 $\pm$0.93& 0.2  $\pm$ 0.32& 1.8  $\pm$ 1.29& 0.65$\pm$0.93\\
21& 6.25&    5& 5.62& 7.47 $\pm$ 2.07& 2.89 $\pm$ 1.51&6.71 $\pm$2.11& 5.69 $\pm$ 1.16& 3.99 $\pm$ 1.94&6.11 $\pm$1.89& 4.79 $\pm$ 0.73& 3.92 $\pm$ 1.75& 3.96$\pm$1.35\\
\midrule
Avg. & 3.06&4.03&3.72&3.25&2.7&3.76&2.47& 3.03&3.61&\textbf{2.03}& \textbf{1.65}& \textbf{1.41}\\
\midrule
\midrule
3 & 5.63& 2.5 &    5& 4.81 $\pm$ 1.78& 4.52 $\pm$ 1.59&5.18	 $\pm$2   & 4.7  $\pm$ 1.87& 4.84 $\pm$ 1.67&5.62 $\pm$1.8 & 6.22 $\pm$ 1.02& 4.6  $\pm$ 3.53&5.15 $\pm$2.5 \\
9 &15.63&22.50& 22.5&15    $\pm$ 2.94& 7.94 $\pm$ 2.87&14.31 $\pm$3.15&13.46 $\pm$ 1.66& 8.7  $\pm$ 2.7 &14.39$\pm$2.57&13.44 $\pm$ 1.38& 2.54 $\pm$ 1.30&5.66 $\pm$1.3 \\
15& 1.25&    0&    0& 2.29 $\pm$ 0.95& 1.94 $\pm$ 1.04&2.5	 $\pm$1.42& 2.3  $\pm$ 0.77& 2.11 $\pm$ 1.20&2.47 $\pm$1.28& 2.34 $\pm$ 0.68& 1.85 $\pm$ 1.18&1.35 $\pm$0.86\\
\bottomrule
\end{tabular}
\end{threeparttable}
\end{table*}

\subsubsection{Fault Types}
The description of the fault types is listed in Table \ref{table:faulttypes}. According to the difficulty level for faults to detect, they can be roughly divided into three categories as follows. The first includes the faults with large magnitudes: IDV(1), IDV(2), IDV(6)-IDV(8), IDV(12)-IDV(14), IDV(17), IDV(18), which are usually easily detectable. The second contains the faults: IDV(5), IDV(10), IDV(16), IDV(19)-IDV(21), which are relatively difficult to detect and the focus of detection. The third includes IDV(3), IDV(9) and IDV(15). They are incipient faults and usually excluded to be detected, because their mechanisms and modes are extremely complicated \cite{wangisa2020, zhang2009ces}. In our experiment, their detection results are still given, but they are not adopted for comparison.

\subsubsection{Dataset}
In this simulation, we choose 11 manipulated variables and 22 process variables to form the input matrix $\mathbf{X}$. The training dataset and validation dataset are made up of 1168 normal samples and 292 normal samples respectively. For the test dataset, it consists of 21 faulty sets, each of which includes 960 samples, where the first 160 samples are fault-free and the rest are faulty.

\subsubsection{Evaluation Index}
To evaluate the fault detection performance, two evaluation indexes, including the fault detection rate (FDR) and the false alarm rate (FAR) \cite{ren2020tii} are adopted in our experiments. A model is considered to be better than another model if its FDR is higher and FAR is lower.

\subsubsection{Comparative Method}
Comparative methods include KPCA \cite{lee2004ces}, DAE \cite{zhang2018jpc}. Among these methods, KPCA is a classical kernel-based nonlinear PCA method and DAE is a basic DAE method. In addition, we call the DAE-PCA method without introducing ${\Omega _\mathbf{T}}$ as DAE-PCA-1, which will also be regarded as a contrast method to demonstrate the indispensability of ${\Omega _\mathbf{T}}$ in DAE-PCA-2. Notice that the DAE-PCA method proposed in this paper actually refers to DAE-PCA-2, not DAE-PCA-1.

\subsubsection{Network Setting}
In neural network structure of DAE-PCA-1 and DAE-PCA-2, we utilize ReLU as the activation function. The model parameters are $N = 1168$, $m = 33$, $d = 33$, $a = 30$, ${\lambda _1} = {\left( {N \times m} \right)^{ - 1}} = 2.6 \times {10^{ - 5}}$, ${\lambda _2} = {\left( {N \times d} \right)^{ - 1}} = 2.6 \times {10^{ - 5}}$, ${\lambda _3} = 2/a = 0.067$.
The confidence level is $\alpha {\rm{ = }}0.99$. Besides, the Adam algorithm \cite{kingma2014} is selected as the optimizer. The maximum $ite{r_{\max }}$ of training iteration ($iter$) is set to $2 \times {10^4}$ and the learning rate ($lr$) is set piecewisely according to $iter$:
\begin{equation}
	lr = 0.01 \times {0.7^{\left\lfloor {iter/350} \right\rfloor }}{\kern 1pt} {\kern 1pt} {\kern 1pt} \left( {0 \le iter \le ite{r_{\max }}} \right).
\end{equation}
At the end of iteration, the model corresponding to the minimum reconstruction error of the verification set will be chosen as the optimal system model.

Here, the parameters of other methods are also given. For KPCA, it adopts RBF kernel and kernel parameter is $5\sqrt {330} \left( { \approx 90.83} \right)$. The number of the principal components is $30$. The network structure of DAE is the same as that of DAE-PCA-2.

\subsection{Implementation Results}
Since DAE, DAE-PCA-1 and DAE-PCA-2 all belong to the neural network approaches, their detection results will be affected by model initialization. Therefore, the test sets are independently tested $50$ times using these methods and then calculate the mean value and standard deviation of FDRs and FARs of statistics for the subsequent comparison.

The FDRs for 21 faults using all methods are shown in Table \ref{table:fdrresults}. For faults in the first category, the FDRs of these four methods are almost over $90\%$ in $\mathcal{PS}$, $\mathcal{RS}$ and $\mathcal{FS}$. It indicates that all comparative methods receive similar and satisfactory detection results. As for faults in the second category, DAE-PCA-2 have more superior detection performance for two subspaces and full space than other methods, due to its higher mean values of FDRs. Besides, DAE-PCA-2 also has smaller variances for most faults than other DAE-based methods, such as IDV(4) and IDV(16). It greatly reflects that DAE-PCA-2 with regularization term ${\Omega _\mathbf{T}}$ has higher stability than DAE-PCA-1 without it and DAE. In contrast, DAE and KPCA have relatively mediocre fault detection results. In addition, in order to further reflect the average performance of various methods for different faults, Avg. is given as the average of the mean of statistics for all faults in Table \ref{table:fdrresults}. Because all methods fail to detect the third category of faults, the calculation of Avg. does not include these faults. The average values of FDRs in DAE-PCA-2 method are also the highest, which indicates that DAE-PCA-2 method can produce a better mathematical expectation in FDRs.

Table \ref{table:farresults} displays the FARs for 21 faults using all methods. From Table \ref{table:farresults}, DAE-PCA-2 has the lowest FARs and Avg. of FARs among all nonlinear methods in both subspaces and the full space, for instance, IDV(16) and IDV(21). The results tell that DAE-PCA-2 is not prone to false alarm for normal samples, thereby reducing manual useless checks.

\begin{figure}[!t]
	\setlength{\abovecaptionskip}{0pt}
	\setlength{\belowcaptionskip}{0pt}
	\centering
	\includegraphics[width=0.7\columnwidth]{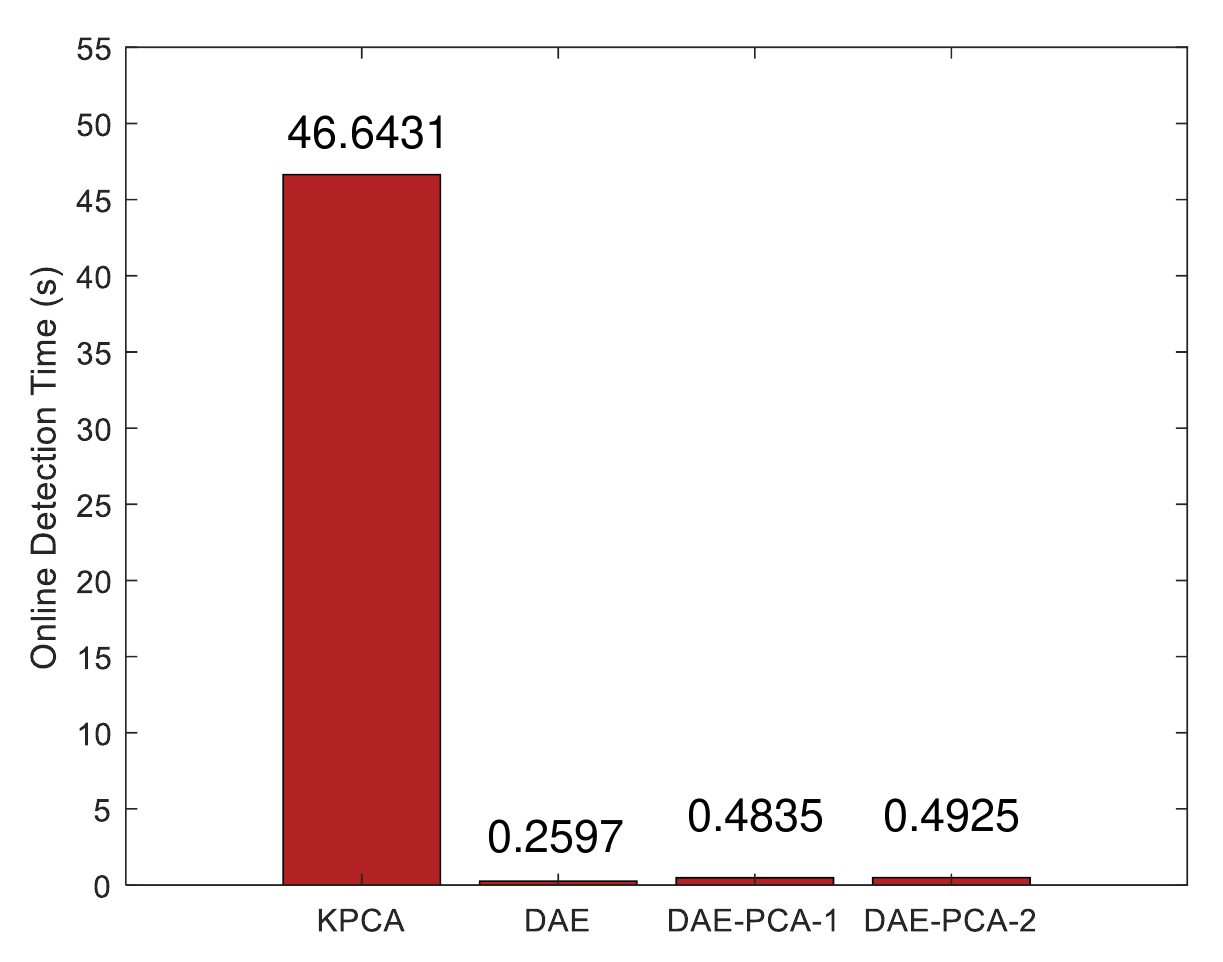}
	\caption{The online detection time (s) of all methods. Here, the detection time refers to all time of 960 samples in a fault.}
	\label{fig:time}
\end{figure}

To demonstrate that DAE-PCA method is a faster KPCA, the following experiment is given. In practical application, we usually do not pay much attention to the training time. But the detection time is an important indicator of real-time performance, so it as a key factor for the quality of the detection method needs to be considered. Fig. \ref{fig:time} lists the online detection time $t$ for all comparative methods. The online detection cost of KPCA is too heavy due to its kernel calculation that makes KPCA hard to apply in real systems. In comparison to KPCA, DAE-based methods including DAE, DAE-PCA-1 and DAE-PCA-2 need less detection time. Therefore, DAE-PCA-2 can meet the real-time requirements, but KPCA cannot.

In addition, by calculation, we get ${\left\| {\mathbf{P}^{\top}}\mathbf{P} - {\mathbf{I}_a} \right\|_F^2} = 6.49 \times {10^{ - 15}}$ from orthogonal projection matrix $\mathbf{P}$. The result demonstrates that the hard constraint given by our PCA module has a quite precise orthogonal property.

Through the above analysis, we can know that DAE-PCA-2, as a learnable KPCA, has the same or even better detection results than KPCA for different kinds of faults. Meanwhile, DAE-PCA-2, also as a faster KPCA, enables to better satisfy the real-time requirements than KPCA. Besides, DAE-PCA-2 has better detection performance and stability than DAE, which suggests the proposed DAE-FE nonlinear framework is much efficient.

In order to observe the situation of the statistic $BIC$ in $\mathcal{FS}$ of the sample in each moment, IDV(5) and IDV(20) are selected as the examples to illustrate the efficiency of the proposed DAE-PCA in detail. Without loss of generality, we randomly selected a trial to demonstrate the detection performance of all approaches based on neural networks.

IDV(5) is a step fault and the detection results for it is displayed in Fig. \ref{fig:Fault5}. The statistic $BIC$ can approximate the probability that the sample is faulty. The greater the value of statistic $BIC$ is, the greater probability the sample has to be considered as a faulty sample. From Fig. \ref{fig:Fault5}, for faulty samples, the statistic $BIC$ values of KPCA and DAE fluctuate between $0.6$ and $1$, while those of DAE-PCA-1 and DAE-PCA-2 remain almost constant at 1. Regarding the samples in the normal state, many statistic $BIC$ values by using DAE and DAE-PCA-1 exceed the threshold $0.01$, while the statistic $BIC$ indexes of KPCA has lower values than DAE and DAE-PCA-1. But DAE-PCA-2 has the lowest values among these methods. Besides, these four methods are all without detection delays.

\begin{figure}[!t]
	\setlength{\abovecaptionskip}{0pt}
	\setlength{\belowcaptionskip}{0pt}
	\centering
	\includegraphics[width=1\columnwidth]{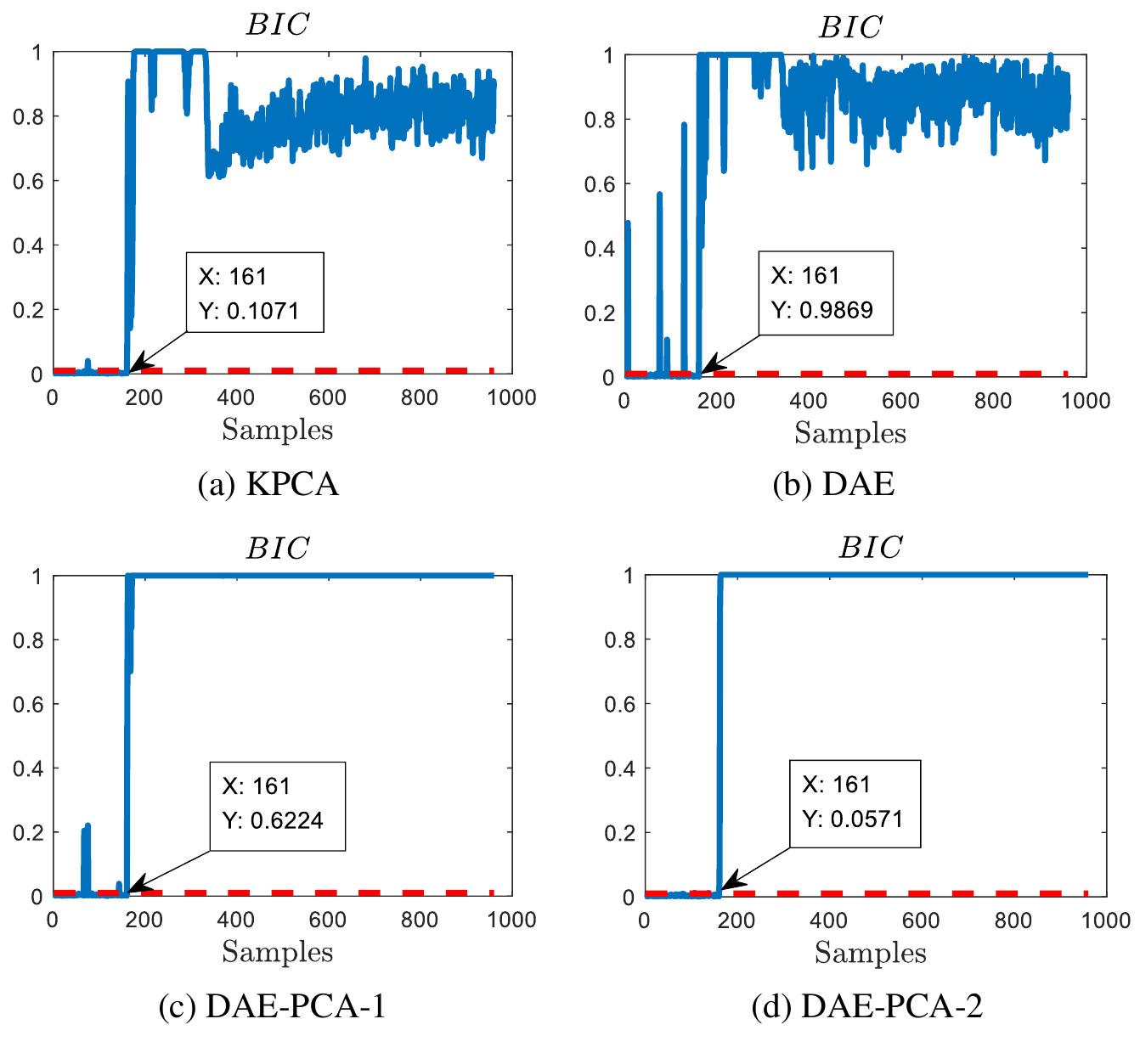}
	\caption{Detection results of Fault IDV(5) in the TE process: (a) KPCA, (b) DAE, (c) DAE-PCA-1 and (d) DAE-PCA-2.}
	\label{fig:Fault5}
\end{figure}

\begin{figure}[!t]
	\setlength{\abovecaptionskip}{0pt}
	\setlength{\belowcaptionskip}{0pt}
	\centering
	\includegraphics[width=1\columnwidth]{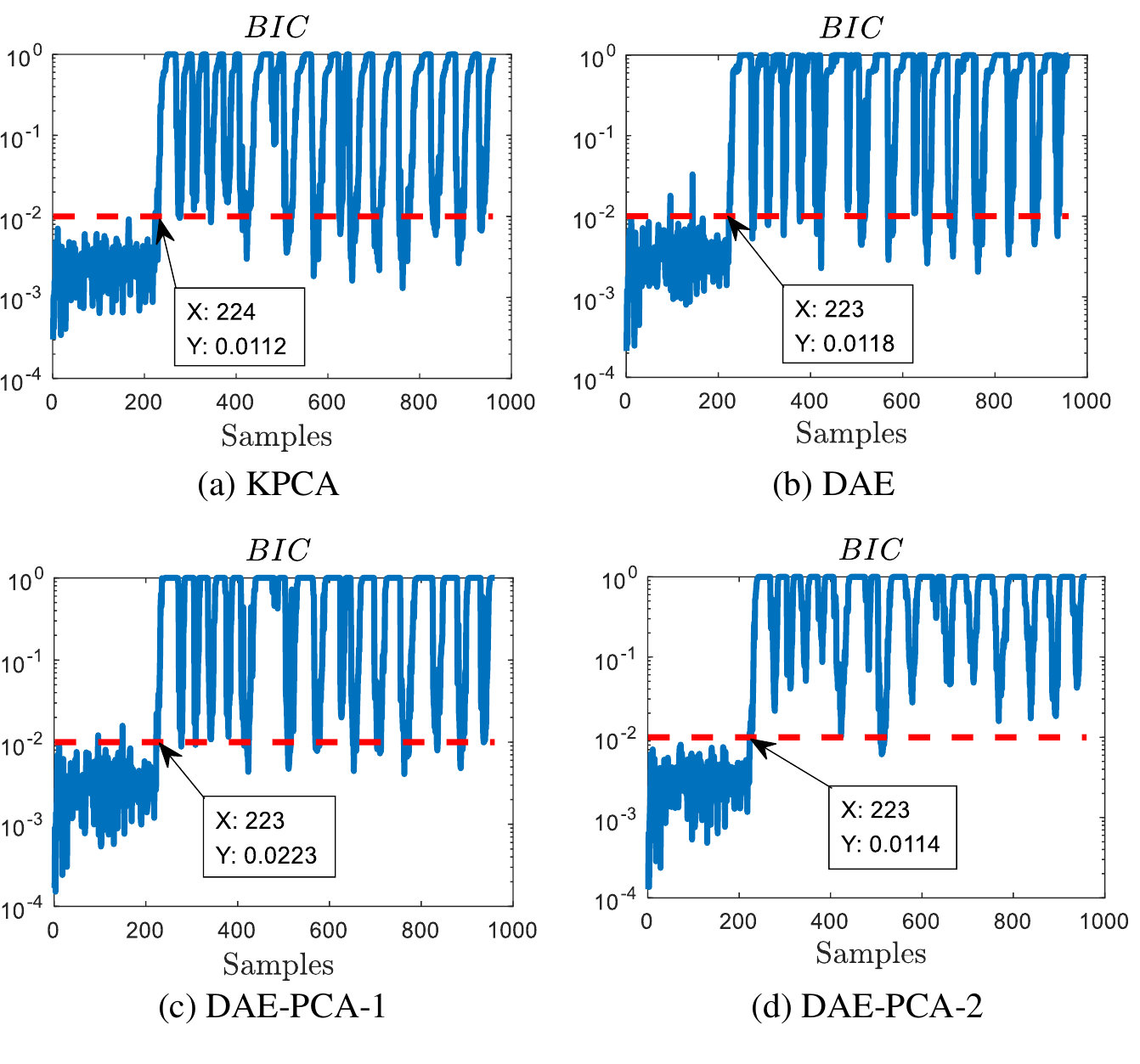}
	\caption{Detection results of Fault IDV(20) in the TE process: (a) KPCA, (b) DAE, (c) DAE-PCA-1 and (d) DAE-PCA-2.}
	\label{fig:Fault20}
\end{figure}

IDV(20) is a type-unknown fault. From Fig. \ref{fig:Fault20}, the detection results of all methods are provided. KPCA has relatively bad FDRs for statistics $BIC$ in faulty samples, while the performance of DAE and DAE-PCA-1 are at an average level. By contrast, DAE-PCA-2 shows outstanding results since its $BIC$ indexes are almost far above the thresholds in almost all faulty sample points. For normal samples, DAE performs unsatisfactorily since its $BIC$ in samples made missed alarms are much higher than the threshold, which is undesirable. KPCA, DAE-PCA-1 and DAE-PCA-2 are all satisfactory and far below the thresholds when the samples are free-fault. In detection delays, the delay time of all methods is nearly the same and is 63 sample points.

Through the analysis of the above two faults, it implies that the proposed DAE-PCA-2 has better separable property to handle the binary classification problem that whether the system is in a normal state or a faulty state, because they have has a larger margin. Furthermore, it also strongly indicates that the kernel learned by DAE-PCA-2 is more suitable for fault detection than that obtained by the predetermined kernel method.

\section{Conclusion}
In this paper, we present a DAE-PCA method based on a nonlinear DAE-FE framework that provides a way to transform a linear approach into the corresponding nonlinear version for fault detection tasks. The DAE-FE framework has been proven to equal to a corresponding learnable and faster kernel trick. In the case of fixed neural network structure, DAE-FE can automatically learn network weights that function similarly as the kernel parameters. Under this framework, we further design DAE-FE framework to put forward a DAE-PCA approach. Owing to Cayley Transform, the orthogonal projection matrix of PCA module in DAE-PCA method is guaranteed by hard constraints, which is implemented by the network structure. Furthermore, the proposed DAE-PCA method contains a regularization term that can reduce the variance of extracted system features, in order to make the normal samples more compact in the feature space. In comparison with KPCA, the proposed DAE-PCA has the ability to automatically learn an optimal parameter and has the faster computational efficiency. Through the experiment, the effectiveness of the proposed approach has been validated by detecting the faults in the TE process.

\end{document}